\documentclass{article}
\PassOptionsToPackage{square,sort,comma,numbers}{natbib}

\usepackage[final]{nips_2018}

\usepackage[utf8]{inputenc} 
\usepackage[T1]{fontenc}    
\usepackage[colorlinks=true,allcolors=black!40!blue]{hyperref}
\usepackage{url}            
\usepackage{booktabs}       
\usepackage{amsfonts}       
\usepackage{nicefrac}       
\usepackage{microtype}      
\usepackage[disable]{todonotes}

\usepackage{mycustom} 
\usepackage{mathrsfs}

\title{Amortized Inference Regularization}

\author{
  Rui Shu\\
  Stanford University\\
  \texttt{ruishu@stanford.edu}
  \And
  Hung H. Bui\\
  DeepMind\\
  \texttt{buih@google.com}
  \And
  Shengjia Zhao\\
  Stanford University\\
  \texttt{sjzhao@stanford.edu}\\
  \And
  Mykel J. Kochenderfer\\
  Stanford University\\
  \texttt{mykel@stanford.edu}
  \And
  Stefano Ermon\\
  Stanford University\\
  \texttt{ermon@cs.stanford.edu}
}

\begin{document}

\maketitle

\begin{abstract}
The variational autoencoder (VAE) is a popular model for density estimation and representation learning. 
Canonically, the variational principle suggests to prefer an expressive inference model so that the variational approximation is accurate. 
However, it is often overlooked that an overly-expressive inference model can be detrimental to the test set performance of both the amortized posterior approximator and, more importantly, the generative density estimator.
In this paper, we leverage the fact that VAEs rely on \emph{amortized} inference and propose techniques for \emph{amortized inference regularization} (AIR) that control the smoothness of the inference model. We demonstrate that, by applying AIR, it is possible to improve VAE generalization on both inference and generative performance. 
Our paper challenges the belief that amortized inference is simply a mechanism for approximating maximum likelihood training and illustrates that regularization of the amortization family provides a new direction for understanding and improving generalization in VAEs.

\end{abstract}

\section{Introduction}


Variational autoencoders are a class of generative models with widespread applications in density estimation, semi-supervised learning, and representation learning \cite{kingma2013auto,kingma2014semi,kim2018disentangling,chen2018isolating}. A popular approach for the training of such models is to maximize the log-likelihood of the training data. However, maximum likelihood is often intractable due to the presence of latent variables. Variational Bayes resolves this issue by constructing a tractable lower bound of the log-likelihood and maximizing the lower bound instead. Classically, Variational Bayes introduces per-sample approximate proposal distributions that need to be optimized using a process called variational inference. However, per-sample optimization incurs a high computational cost. A key contribution of the variational autoencoding framework is the observation that the cost of variational inference can be amortized by using an amortized inference model that learns an efficient mapping from samples to proposal distributions. This perspective portrays amortized inference as a tool for efficiently approximating maximum likelihood training. Many techniques have since been proposed to expand the expressivity of the amortized inference model in order to better approximate maximum likelihood training \cite{kim2018semi,kingma2016improved,sonderby2016ladder,burda2015importance}.

In this paper, we challenge the conventional role that amortized inference plays in variational autoencoders. For datasets where the generative model is prone to overfitting, we show that having an amortized inference model actually provides a new and effective way to regularize maximum likelihood training. Rather than making the amortized inference model more expressive, we propose instead to restrict the capacity of the amortization family. Through amortized inference regularization (AIR), we show that it is possible to reduce the inference gap and increase the log-likelihood performance on the test set. We propose several techniques for AIR and provide extensive theoretical and empirical analyses of our proposed techniques when applied to the variational autoencoder and the importance-weighted autoencoder. By rethinking the role of the amortized inference model, amortized inference regularization provides a new direction for studying and improving the generalization performance of latent variable models.

\section{Background and Notation}
\subsection{Variational Inference and the Evidence Lower Bound}

Consider a joint distribution $p_\theta(x, z)$ parameterized by $\theta$, where $x \in \X$ is observed and $z \in \Z$ is latent. Given a uniform distribution $\phat(x)$ over the dataset $\D = \seta{x^{(i)}}$, maximum likelihood estimation performs model selection 
using the objective
\begin{align}
\max_{\theta} \Expect_{\phat(x)} \ln p_\theta(x) 
= \max_{\theta} \Expect_{\phat(x)} \ln \int_z p_\theta(x, z) dz.
\end{align}
However, marginalization of the latent variable is often intractable; to address this issue, it is common to employ the variational principle to maximize the following lower bound
\begin{align}
\max_\theta \Expect_{\phat(x)} \brac{\ln p_\theta(x) - \min_{q \in \Q} \KL{q(z)}{p_\theta(z \giv x)}} = \max_\theta \Expect_{\phat(x)} \brac{\max_{q \in \Q}\Expect_{q(z)} \ln \frac{p_\theta(x, z)}{q(z)}}, \label{eq:elbo}
\end{align}
where $D$ is the Kullback-Leibler divergence and $\Q$ is a variational family. This lower bound, commonly called the evidence lower bound (ELBO), converts log-likelihood estimation into a tractable
optimization problem. Since the lower bound holds for any $q$, the variational family $\Q$ can be chosen to ensure that $q(z)$ is easily computable, and the lower bound is optimized to select the best proposal distribution $q^*_x(z)$ for each $x \in \D$.

\subsection{Amortization and Variational Autoencoders}
\cite{kingma2013auto,rezende2014stochastic} proposed to construct $p(x \giv z)$ using a parametric function $g_\theta \in \G(\P): \Z \to \P$,
where $\P$ is some family of distributions over $x$, and $\G$ is a family of functions indexed by parameters $\theta$.
To expedite training, they observed that it is possible to amortize the computational cost of variational inference by framing the per-sample optimization process as a \emph{regression} problem; rather than solving for the optimal proposal $q^*_x(z)$ directly, they instead use a recognition model $f_\phi \in \F(\Q): \X \to \Q$ to predict $q^*_x(z)$. 
The functions $(f_\phi, g_\theta)$ can be concisely represented as conditional distributions, where 
\begin{align}
p_\theta(x \giv z) &= g_\theta(z)(x) \\
q_\phi(z \giv x) &= f_\phi(x)(z).
\end{align}
The use of amortized inference yields the variational autoencoder, which is trained to maximize the variational autoencoder objective
\begin{align}
\max_{\theta, \phi} \Expect_{\phat(x)} \brac{\Expect_{q_\phi(z \giv x)} \ln \frac{p(z)p_\theta(x \giv z)}{q_\phi(z \giv x)}} 
= \max_{f \in \F(\Q), g \in \G(\P)} 
\Expect_{\phat(x)} 
\brac{\Expect_{z \sim f(x)} \ln \frac{p(z)g(z)(x)}{f(x)(z)}}. \label{eq:vae}
\end{align}
We omit the dependency of $(p(z), g)$ on $\theta$ and $f$ on $\phi$ for notational simplicity. In addition to the typical presentation of the variational autoencoder objective (LHS), we also show an alternative formulation (RHS) that reveals the influence of the model capacities $\F, \G$ and distribution family capacities $\Q, \P$ on the objective function. In this paper, we use $(q_\phi, f)$ interchangeably, depending on the choice of emphasis. To highlight the relationship between the ELBO in \cref{eq:elbo} and the standard variational autoencoder objective in \cref{eq:vae}, we shall also refer to the latter as the amortized ELBO.

\subsection{Amortized Inference Suboptimality}
For a fixed generative model, the optimal unamortized and amortized inference models are
\begin{align}
q^*_x &= \argmax_{q \in \Q} \Expect_{q(z)} \brac{\ln \frac{p_\theta(x, z)}{q(z)}}, \text{for each } x \in \D \label{eq:optimalq}\\
f^* &= \argmax_{f \in \F} \Expect_{\phat(x)} \brac{\Expect_{z\sim f(x)} \ln \frac{p_\theta(x, z)}{f(x)(z)}}.
\end{align}

A notable consequence of using an amortization family to approximate variational inference is that \cref{eq:vae} is a lower bound of \cref{eq:elbo}. This naturally raises the question of whether the learned inference model can accurately approximate the mapping $x \mapsto q^*_x(z)$. To address this question, \cite{cremer2018inference} defined the inference, approximation, and amortization gaps as
\begin{align}
\Delta_\infer(\phat) &= \Expect_{\phat(x)}\KL{f^*(x)}{p_\theta(z \giv x)}\\
\Delta_\ap(\phat) &= \Expect_{\phat(x)}\KL{q^*_x(z)}{p_\theta(z \giv x)}\\
\Delta_\am(\phat) &= \Delta_\infer(\phat) - \Delta_\ap(\phat),
\end{align}
Studies have found that the inference gap is non-negligible \cite{wu2016quantitative} and primarily attributable to the presence of a large amortization gap \cite{cremer2018inference}.

The amortization gap raises two critical considerations. On the one hand, we wish to reduce the training amortization gap $\Delta_\am(\phat_\tr)$. If the family $\F$ is too low in capacity, then it is unable to approximate $x \mapsto q^*_x$ and will thus increase the amortization gap. Motivated by this perspective, \cite{kim2018semi,krishnan2017challenges} proposed to reduce the training amortization gap by performing stochastic variational inference on top of amortized inference. In this paper, we take the opposing perspective that an over-expressive $\F$ hurts generalization (see \cref{app:large_encoder}) and that restricting the capacity of $\F$ is a form of regularization that can prevent both the inference \emph{and} generative models from overfitting to the training set. 





\section{Amortized Inference Regularization in Variational Autoencoders}

Many methods have been proposed to expand the variational and amortization families in order to better approximate maximum likelihood training \cite{kim2018semi,kingma2016improved,sonderby2016ladder,burda2015importance,maaloe2016auxiliary,ranganath2016hierarchical}. We argue, however, that achieving a better approximation to maximum likelihood training is not necessarily the best training objective, even if the end goal is test set density estimation. In general, it may be beneficial to regularize the maximum likelihood training objective. 

Importantly, we observe that the evidence lower bound in \cref{eq:elbo} admits a natural interpretation as implicitly regularizing maximum likelihood training
\begin{align}
\max_{\theta} \parenb{\overbrace{\Expect_{\phat(x)}\brac{\ln p_\theta(x)}}^{\text{log-likelihood}} - \overbrace{\Expect_{\phat(x)}\min_{q \in \Q}{\KL{q(z)}{p_\theta(z \giv x)}}}^{\text{regularizer}~ R(\theta\scolon \Q)}}.
\end{align}

This formulation exposes the ELBO as a \emph{data-dependent regularized} maximum likelihood objective. For infinite capacity $\Q$, $R(\theta\scolon \Q)$ is zero for all $\theta \in \Theta$, and the objective reduces to maximum likelihood. When $\Q$ is the set of Gaussian distributions (as is the case in the standard VAE), then $R(\theta\scolon \Q)$ is zero only if $p_\theta(z \giv x)$ is Gaussian for all $x \in \D$. In other words, a Gaussian variational family regularizes the true posterior $p_\theta(z \giv x)$ toward being Gaussian \cite{cremer2018inference}. Careful selection of the variational family to encourage $p_\theta(z \giv x)$ to adopt certain properties (e.g. unimodality, fully-factorized posterior, etc.) can thus be considered a special case of \emph{posterior regularization} \cite{ganchev2010posterior,zhu2014bayesian}.


Unlike traditional variational techniques, the variational autoencoder introduces an amortized inference model $f \in \F$ and thus a new source of posterior regularization. 
\begin{align}
\max_{\theta} \parenb{\overbrace{\Expect_{\phat(x)}\brac{\ln p_\theta(x)}}^{\text{log-likelihood}} - \overbrace{\min_{f \in \F(\Q)}\Expect_{\phat(x)}\brac{\KL{f(x)}{p_\theta(z \giv x)}}}^{\text{regularizer}~ R(\theta\scolon \Q, \F)}}.
\end{align}
In contrast to unamortized variational inference, the introduction of the amortization family $\F$ forces the inference model to consider the \emph{global structure} of how $\X$ maps to $\Q$. We thus define \emph{amortized inference regularization} as the strategy of restricting the inference model capacity $\F$ to satisfy certain desiderata. In this paper, we explore a special case of AIR where a candidate model $f \in \F$ is penalized if it is not sufficiently smooth. We propose two models that encourage inference model smoothness and demonstrate that they can reduce the inference gap and increase log-likelihood on the test set. 

\subsection{Denoising Variational Autoencoder}

In this section, we propose using random perturbation training for amortized inference regularization. The resulting model\textemdash the denoising variational autoencoder (DVAE)\textemdash modifies the variational autoencoder objective by injecting $\veps$ noise into the inference model
\begin{align}\label{eq:dvae}
\max_{\theta} \parenb{\Expect_{\phat(x)}\brac{\ln p_\theta(x)} - \min_{f \in \F(\Q)}\Expect_{\phat(x)}\Expect_{\veps}\brac{\KL{f(x + \veps)}{p_\theta(z \giv x)}}}.
\end{align}
Note that the noise term only appears in the regularizer term. We consider the case of zero-mean isotropic Gaussian noise $\veps \sim \Normal(\0, \sigma\I)$ and denote the denoising regularizer as $R(\theta\scolon \sigma)$. At this point, we note that the DVAE was first described in \cite{im2017denoising}. However, our treatment of DVAE differs from \cite{im2017denoising}'s in both theoretical analysis and underlying motivation. We found that \cite{im2017denoising} incorrectly stated the tightness of the DVAE variational lower bound (see \cref{app:dvae}). In contrast, our analysis demonstrates that the denoising objective smooths the inference model and necessarily lower bounds the original variational autoencoder objective (see \cref{thm:smooth,prop:strength}). 

We now show that 1) the optimal DVAE amortized inference model is a kernel regression model and that 2) the variance of the noise $\veps$ controls the smoothness of the optimal inference model.

\begin{restatable}{lemma}{PropKernel}\label{prop:kernel}
For fixed $(\theta, \sigma, \Q)$ and infinite capacity $\F$, the inference model that optimizes the DVAE objective in \cref{eq:dvae} is the kernel regression model
\begin{align}
f^*_\sigma(x) = \argmin_{q \in \Q} \sum_{i=1}^n w_\sigma(x, x^{(i)}) \cdot \KL{q(z)}{p_\theta(z \giv x^{(i)})},\label{eq:kernel}
\end{align}
where $w_\sigma(x, x^{(i)}) = \frac{K_\sigma(x, x^{(i)})}{\sum_j K_\sigma(x, x^{(j)})}$ and $K_\sigma(x, y) = \exp\paren{-\frac{\|x - y\|}{2\sigma^2}}$ is the RBF kernel.
\end{restatable}

\cref{prop:kernel} shows that the optimal denoising inference model $f^*_\sigma$ is dependent on the noise level $\sigma$. The output of $f^*_\sigma(x)$ is the proposal distribution that minimizes the weighted Kullback-Leibler (KL) divergence from $f^*_\sigma(x)$ to each $p_\theta(z \giv x^{(i)})$, where the weighting $w_\sigma(x, x^{(i)})$ depends on the distance $\|x - x^{(i)}\|$ and the bandwidth $\sigma$. When $\sigma > 0$, the amortized inference model forces neighboring points $(x^{(i)}, x^{(j)})$ to have similar proposal distributions. Note that as $\sigma$ increases, $w_\sigma(x, x^{(i)}) \to \frac{1}{n}$, where $n$ is the number of training samples. Controlling $\sigma$ thus modulates the smoothness of $f_\sigma^*$ (we say that $f_\sigma^*$ is smooth if it maps similar inputs to similar outputs under some suitable measure of similarity). Intuitively, the denoising regularizer $R(\theta \scolon \sigma)$ approximates the true posteriors with a ``$\sigma$-smoothed'' inference model and penalizes generative models whose posteriors cannot easily be approximated by such an inference model. This intuition is formalized in \cref{thm:smooth}.

\begin{restatable}{theorem}{ThmSmooth}\label{thm:smooth}
Let $\Q$ be a minimal exponential family with corresponding natural parameter space $\Omega$. With a slight abuse of notation, consider $f \in \F: \X \to \Omega$. Under the simplifying assumption that $p_\theta(z \giv x^{(i)})$ is contained within $\Q$ and parameterized by $\eta^{(i)} \in \Omega$, and that $\F$ has infinite capacity, then the optimal inference model in \cref{prop:kernel}
returns $f^*_\sigma(x) = \eta\in\Omega$, where
\begin{align}
\eta = \sum_{i=1}^n w_\sigma(x, x^{(i)}) \cdot \eta^{(i)}
\end{align}
and Lipschitz constant of $f^*_\sigma$ is bounded by $O(1/\sigma^2)$.
\end{restatable}

We wish to address \cref{thm:smooth}'s assumption that the true posteriors lie in the variational family. Note that for sufficiently large exponential families, this assumption is likely to hold. But even in the case where the variational family is Gaussian (a relatively small exponential family), the small approximation gap observed in \citep{cremer2018inference} suggests that it is plausible that posterior regularization would encourage the true posteriors to be approximately Gaussian.

Given that $\sigma$ modulates the smoothness of the inference model, it is natural to suspect that a larger choice of $\sigma$ results in a stronger regularization. To formalize this notion of regularization strength, we introduce a way to partially order a set of regularizers $\set{R_i(\theta)}$. 

\begin{definition}\label{def:strength}
Suppose two regularizers $R_1(\theta)$ and $R_2(\theta)$ share the same minimum $\min_\theta R_1(\theta) = \min_\theta R_2(\theta)$. We say that $R_1$ is a stronger regularizer than $R_2$ if 
$R_1(\theta) \ge R_2(\theta)$ for all $\theta \in \Theta$. 
\end{definition}

Note that any two regularizers can be modified via scalar addition to share the same minimum. Furthermore, if $R_1$ is stronger than $R_2$, then $R_1$ and $R_2$ share at least one minimizer. We now apply \cref{def:strength} to characterize the regularization strength of $R(\theta \scolon \sigma)$ as $\sigma$ increases.

\begin{definition}
We say that $\F$ is closed under input translation if $f \in \F \implies f_a \in \F$ for all $a \in \X$,  where $f_a(x) = f(x + a)$.
\end{definition}

\begin{restatable}{proposition}{PropStrength}\label{prop:strength}
Consider the denoising regularizer $R(\theta \scolon \sigma)$. Suppose $\F$ is closed under input translation and that, for any $\theta \in \Theta$, there exists $f \in \F$ such that $f(x)$ maps to the prior $p_\theta(z)$ all $x \in \X$. Furthermore, assume that there exists $\theta \in \Theta$ such that $p_\theta(x, z) = p_\theta(z)p_\theta(x)$. Then $R(\theta \scolon \sigma_1)$ is stronger $R(\theta \scolon \sigma_2)$ when $\sigma_1 \ge \sigma_2$; i.e., $\min_\theta R(\theta \scolon \sigma_1) = \min_\theta R(\theta \scolon \sigma_2) = 0$ and $R(\theta \scolon \sigma_1) \ge R(\theta \scolon \sigma_2)$ for all $\theta \in \Theta$.
\end{restatable}


\cref{prop:kernel,prop:strength} show that as we increase $\sigma$, the optimal inference model is forced to become smoother and the regularization strength increases. \Cref{fig:strength_ablation_iw1} is consistent with this analysis, showing the progression from under-regularized to over-regularized models as we increase $\sigma$.

It is worth noting that, in addition to adjusting the denoising regularizer strength via $\sigma$, it is also possible to adjust the strength by taking a convex combination of the VAE and DVAE objectives. In particular, we can define the \emph{partially} denoising regularizer $R(\theta\scolon \sigma, \alpha)$ as
\begin{align}
\min_{f \in \F(\Q)}\Expect_{\phat(x)}
\parenc{
\alpha \cdot \Expect_{\veps}\brac{\KL{f(x + \veps)}{p_\theta(z \giv x)}}
+
(1 - \alpha) \cdot {\KL{f(x)}{p_\theta(z \giv x)}}
}
\end{align}
Importantly, we note that $R(\theta \scolon \sigma, \alpha)$ is still strictly non-negative and, when combined with the log-likelihood term, still yields a tractable variational lower bound.


\subsection{Weight-Normalized Amortized Inference}

In addition to DVAE, we propose an alternative method that directly restricts $\F$ to the set of smooth functions. To do so, we consider the case where the inference model is a neural network encoder parameterized by weight matrices $\seta{W_i}$ and leverage \cite{salimans2016weight}'s weight normalization technique, which proposes to reparameterize the columns $w_i$ of each weight matrix $W$ as
\begin{align}
w_i = \frac{v_i}{\| v_i \|} \cdot s_i,
\end{align}
where $v_i \in \R^d, s_i \in \R$ are trainable parameters. Since it is possible to modulate the smoothness of the encoder by capping the magnitude of $s_i$, we introduce a new parameter $u_i \in \R$ and define
\begin{align}
s_i = \min\set{\|v_i \|, \paren{\frac{H}{1 + \exp(-u_i)}}}. 
\end{align}
The norm $\| w_i \|$ is thus bounded by the hyperparameter $H$. We denote the weight-normalized regularizer as $R(\theta \scolon \F_H)$, where $\F_H$ is the amortization family induced by a $H$-weight-normalized encoder. Under similar assumptions as \cref{prop:strength}, it is easy to see that $\min_\theta R(\theta \scolon \F_H) = 0$ for any $H \ge 0$ and that $R(\theta \scolon \F_{H_1}) \ge R(\theta \scolon \F_{H_2})$ for all $\theta \in \Theta$ when $H_1 \le H_2$ (since $\F_{H_1}\subseteq \F_{H_2}$). We refer to the resulting model as the weight-normalized inference VAE (WNI-VAE) and show in \Cref{table:vae} that weight-normalized amortized inference can achieve similar performance as DVAE.

\subsection{Experiments}

We conducted experiments on statically binarized MNIST, statically binarized OMNIGLOT, and the Caltech 101 Silhouettes datasets. These datasets have a relatively small amount of training data and are thus susceptible to model overfitting. For each dataset, we used the same decoder architecture across all four models (VAE, DVAE $(\alpha=0.5)$, DVAE $(\alpha=1.0)$, WNI-VAE) and only modified the encoder, and trained all models using Adam \cite{kingma2014adam} (see \cref{app:architecture} for more details). To approximate the log-likelihood, we proposed to use importance-weighted stochastic variational inference (IW-SVI), an extension of SVI \cite{hoffman2013stochastic} which we describe in detail in \cref{app:iwsvi}. Hyperparameter tuning of DVAE's $\sigma$ and WNI-VAE's $\F_H$ is described in \cref{table:strengths}.

\begin{table}[!h]
\setlength\tabcolsep{3pt} 
\tiny
\centering
\caption{Test set evaluation of VAE, DVAE, and WNI-VAE. The performance metrics are log-likelihood $\ln p_\theta(x)$, the amortized ELBO $\L(x)$, and the inference gap $\Delta_\infer = \ln p_\theta(x) - \L(x)$. All three proposed models out-perform VAE across most metrics.} \label{table:vae}
\begin{tabular}{l|ccc|ccc|ccc}
\toprule
& \multicolumn{3}{c|}{MNIST} & \multicolumn{3}{c|}{OMNIGLOT} & \multicolumn{3}{c}{CALTECH} \\
& $-\ln p_\theta(x)$ & $\Delta_\infer$ & $-\L(x)$
& $-\ln p_\theta(x)$ & $\Delta_\infer$ & $-\L(x)$
& $-\ln p_\theta(x)$ & $\Delta_\infer$ & $-\L(x)$\\
\midrule
VAE               & 
$86.93$ \scalebox{0.5}{$\pm 0.04$} & $8.54$ \scalebox{0.5}{$\pm 0.14$} & $95.48$ \scalebox{0.5}{$\pm 0.07$} & 
$110.32$ \scalebox{0.5}{$\pm 0.16$} & $12.03$ \scalebox{0.5}{$\pm 0.25$} & $122.35$ \scalebox{0.5}{$\pm 0.33$} & 
$109.14$ \scalebox{0.5}{$\pm 0.28$} & $28.90$ \scalebox{0.5}{$\pm 0.42$} & $138.05$ \scalebox{0.5}{$\pm 0.15$} \\
DVAE $(\alpha=0.5)$ & 
$86.46$ \scalebox{0.5}{$\pm 0.02$} & $\bf6.34$ \scalebox{0.5}{$\pm 0.05$} & $\bf92.80$ \scalebox{0.5}{$\pm 0.07$} & 
$109.31$ \scalebox{0.5}{$\pm 0.19$} & $12.56$ \scalebox{0.5}{$\pm 0.18$} & $121.87$ \scalebox{0.5}{$\pm 0.37$} & 
$\bf108.64$ \scalebox{0.5}{$\pm 0.19$} & $\bf23.40$ \scalebox{0.5}{$\pm 0.19$} & $\bf132.04$ \scalebox{0.5}{$\pm 0.37$} \\
DVAE $(\alpha=1.0)$ & 
$86.51$ \scalebox{0.5}{$\pm 0.02$} & $6.83$ \scalebox{0.5}{$\pm 0.04$} & $93.35$ \scalebox{0.5}{$\pm 0.06$} & 
$110.12$ \scalebox{0.5}{$\pm 0.18$} & $12.44$ \scalebox{0.5}{$\pm 0.16$} & $122.56$ \scalebox{0.5}{$\pm 0.34$} & 
$108.66$ \scalebox{0.5}{$\pm 0.23$} & $23.94$ \scalebox{0.5}{$\pm 0.15$} & $132.60$ \scalebox{0.5}{$\pm 0.15$} \\
WNI-VAE           & 
$\bf86.42$ \scalebox{0.5}{$\pm 0.01$} & $6.68$ \scalebox{0.5}{$\pm 0.01$} & $93.10$ \scalebox{0.5}{$\pm 0.02$} & 
$\bf109.16$ \scalebox{0.5}{$\pm 0.12$} & $\bf11.39$ \scalebox{0.5}{$\pm 0.10$} & $\bf120.55$ \scalebox{0.5}{$\pm 0.20$} 
& $108.94$ \scalebox{0.5}{$\pm 0.31$} & $28.88$ \scalebox{0.5}{$\pm 0.29$} & $137.82$ \scalebox{0.5}{$\pm 0.25$} \\
\midrule
\end{tabular}
\end{table}
\begin{figure}[h]
\centering
\includegraphics[width=\textwidth]{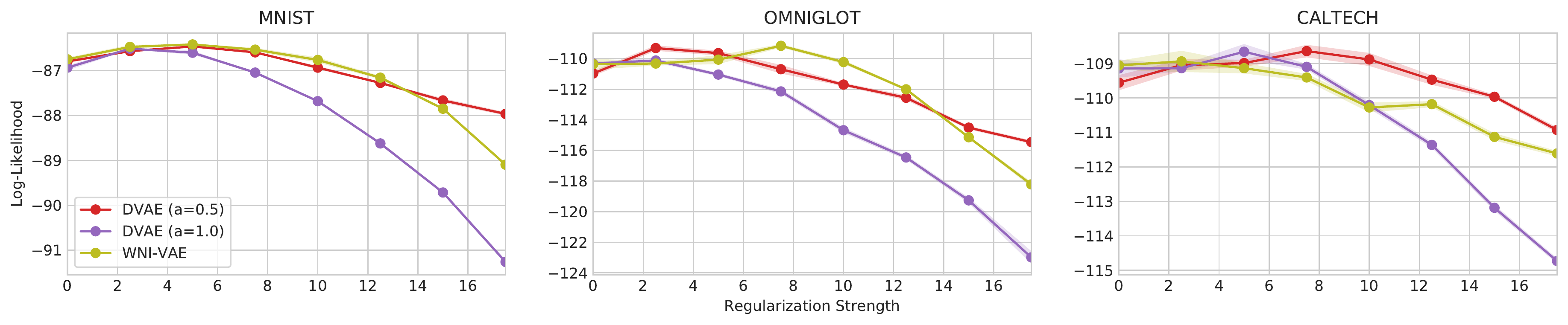}
\caption{Evaluation of the log-likelihood performance of all three proposed models as we vary the regularization parameter value. The regularization parameter is defined in \cref{table:strengths}. When the parameter value is too small, the model overfits and the test set performance degrades. When the parameter value is too high, the model underfits.
}
\label{fig:strength_ablation_iw1}
\end{figure}

\Cref{table:vae} shows the performance of VAE, DVAE, and WNI-VAE. Regularizing the inference model consistently improved the test set log-likelihood performance. On the MNIST and Caltech 101 Silhouettes datasets, the results also show a consistent reduction of the test set inference gap when the inference model is regularized. We observed differences in the performance of DVAE versus WNI-VAE on the Caltech 101 Silhouettes dataset, suggesting a difference in how denoising and weight normalization regularizes the inference model; an interesting consideration would thus be to combine DVAE and WNI. As a whole, \Cref{table:vae} demonstrates that AIR benefits the generative model.

The denoising and weight normalization regularizers have respective hyperparameters $\sigma$ and $H$ that control the regularization strength. In \Cref{fig:strength_ablation_iw1}, we performed an ablation analysis of how adjusting the regularization strength impacts the test set log-likelihood. In almost all cases, we see a transition from overfitting to underfitting as we adjust the strength of AIR.
For well-chosen regularization strength, however, it is possible to increase the test set log-likelihood performance by $0.5 \sim 1.0$ nats\textemdash a non-trivial improvement.






\subsection{How Does Amortized Inference Regularization Affect the Generator?}

\Cref{table:vae} shows that regularizing the inference model empirically benefits the generative model. We now provide some initial theoretical characterization of how a smoothed amortized inference model affects the generative model. Our analysis rests on the following proposition.

\begin{restatable}{proposition}{PropDecoder}\label{prop:decoder}
Let $\P$ be an exponential family with corresponding mean parameter space $\M$ and sufficient statistic function $T(\cdot)$. With a slight abuse of notation, consider $g \in \G: \Z \to \M$. Define $q(x, z) = \phat(x)q(z \giv x)$, where $q(z \giv x)$ is a fixed inference model. Supposing $\G$ has infinite capacity, then the optimal generative model in \cref{eq:vae} returns $g^*(z) = \mu \in \M$, where
\begin{align}
\mu = \sum_{i=1}^n q(x^{(i)} \giv z) \cdot T(x^{(i)}) = \sum_{i=1}^n \paren{\frac{q(z \giv x^{(i)})}{\sum_j q(z \giv x^{(j)})} \cdot T(x^{(i)})}.
\end{align}
\end{restatable}

\cref{prop:decoder} generalizes the analysis in \cite{bousquet2017optimal} which determined the optimal generative model when $\P$ is Gaussian. The key observation is that the optimal generative model outputs a convex combination of $\seta{\phi(x^{(i)})}$, weighted by $q(x^{(i)} \giv z)$. Furthermore, the weights $q(x^{(i)} \giv z)$ are simply density ratios of the proposal distributions $\seta{q(z \giv x^{(i)})}$. As we increase the smoothness of the amortized inference model, the weight $q(x^{(i)} \giv z)$ should tend toward $\frac{1}{n}$ for all $z \in \Z$. This suggests that a smoothed inference model provides a natural way to smooth (and thus regularize) the generative model. \todo{maybe prove this}

\section{Amortized Inference Regularization in Importance-Weighted Autoencoders}
In this section, we extend AIR to importance-weighted autoencoders (IWAE-$k$). Although the application is straightforward, we demonstrate a noteworthy relationship between the number of importance samples $k$ and the effect of AIR. To begin our analysis, we consider the IWAE-$k$ objective
\begin{align}
\max_{\theta,\phi} \Expect_{z_1\ldots z_k \sim q_\phi(z \giv x)} \brac{\ln \frac{1}{k} \sum_{i=1}^k \frac{p_\theta(x, z_i)}{q_\phi(z_i \giv x)}},
\end{align}
where $\seta{z_1\ldots z_k}$ are $k$ samples from the proposal distribution $q_\phi(z \giv x)$ to be used as importance-samples. Analysis by \cite{cremer2017reinterpreting} allows us to rewrite it as a regularized maximum likelihood objective
\begin{align}
\max_\theta \Expect_{\phat(x)}\brac{\ln p_\theta(x)} - 
\overbrace{
\min_{f \in \F(\Q)}
\Expect_{\phat(x)}
\Expect_{z_2\ldots z_k \sim f(x)} 
\UKL{\tilde{f}_k(x, z_1\ldots z_k)}{p_\theta(z \giv x)}}^{R_k(\theta)},
\end{align}
where $\tilde{f}_k$ (or equivalently $\tilde{q}_k$) is the unnormalized distribution 
\begin{align}
\tilde{f}_k(x, z_2\ldots z_k)(z_1) 
= \frac{p_\theta(x, z_1)}{\frac{1}{k} \sum_i \frac{p_\theta(x, z_i)}{f(x)(z_i)}}
=\tilde{q}_k(z_1 \giv x, z_2\ldots z_k) 
\end{align}
and $\UKL{q}{p} = \int q(z) \brac{\ln q(z) - \ln p(z)} dz$ is the Kullback-Leibler divergence extended to unnormalized distributions. For notational simplicity, we omit the dependency of $\tilde{f}_k$ on $(z_2\ldots z_k)$. Importantly, \cite{cremer2017reinterpreting} showed that the IWAE with $k$ importance samples drawn from the amortized inference model $f$ is, on expectation, equivalent to a VAE with $1$ importance sample drawn from the more expressive inference model $\tilde{f}_k$.

\subsection{Importance Sampling Attenuates Amortized Inference Regularization}

We now consider the interaction between importance sampling and AIR. We introduce the regularizer $R_k(\theta \scolon \sigma, \F_H)$ as follows
\begin{align}
R_k(\theta \scolon \sigma, \F_H) = 
\min_{f \in \F_H(\Q)}
\Expect_{\phat(x)}
\Expect_{\veps}\Expect_{z_2\ldots z_k \sim f(x + \veps)} 
\UKL{\tilde{f}_k(x + \veps)}{p_\theta(z \giv x)},
\end{align}
which corresponds to a regularizer where weight normalization, denoising, and importance sampling are simultaneously applied. By adapting Theorem 1 from \cite{burda2015importance}, we can show that

\begin{restatable}{proposition}{PropIWAE}\label{prop:iwae}
Consider the regularizer $R_k(\theta \scolon \sigma, \F_H)$. Under similar assumptions as \cref{prop:strength}, then $R_{k_1}$ is stronger than $R_{k_2}$ when $k_1 \le k_2$; i.e., $\min_\theta R_{k_1}(\theta \scolon \sigma, \F_H) = \min_\theta R_{k_2}(\theta \scolon \sigma, \F_H) = 0$ and $R_{k_1}(\theta \scolon \sigma, \F_H) \le R_{k_2}(\theta \scolon \sigma, \F_H)$ for all $\theta \in \Theta$.
\end{restatable}
\todo{maybe show the asymptotic convergence rate}

A notable consequence of \cref{prop:iwae} is that as $k$ increases, AIR exhibits a weaker regularizing effect on the posterior distributions $\seta{p_\theta(z \giv x^{(i)})}$. Intuitively, this arises from the phenomenon that although AIR is applied to $f$, the subsequent importance-weighting procedure can still create a flexible $\tilde{f}_k$. Our analysis thus predicts that AIR is less likely to cause \emph{underfitting} of IWAE-$k$'s generative model as $k$ increases, which we demonstrate in \Cref{fig:strength_ablation}. In the limit of infinite importance samples, we also predict AIR to have zero regularizing effect since $\tilde{f}_\infty$ (under some assumptions) can always approximate any posterior. However, for practically feasible values of $k$, we show in \cref{table:iwae8,table:iwae64} that AIR is a highly effective regularizer.

\subsection{Experiments}
\begin{table}[!h]
\setlength\tabcolsep{3pt} 
\tiny
\centering
\caption{Test set evaluation of the four models when trained with $8$ importance samples. $\L_8(x)$ denotes the amortized ELBO using $8$ importance samples. $\Delta_\infer = \ln p_\theta(x) - \L_8(x)$.} 
\label{table:iwae8}
\begin{tabular}{l|ccc|ccc|ccc}
\toprule
& \multicolumn{3}{c|}{MNIST} & \multicolumn{3}{c|}{OMNIGLOT} & \multicolumn{3}{c}{CALTECH} \\
& $-\ln p_\theta(x)$ & $\Delta_\infer$ & $-\L_8(x)$
& $-\ln p_\theta(x)$ & $\Delta_\infer$ & $-\L_8(x)$
& $-\ln p_\theta(x)$ & $\Delta_\infer$ & $-\L_8(x)$\\
\midrule
IWAE &
$86.21$ \scalebox{0.5}{$\pm 0.01$} & $6.13$ \scalebox{0.5}{$\pm 0.03$} & $92.34$ \scalebox{0.5}{$\pm 0.02$} &
$108.18$ \scalebox{0.5}{$\pm 0.24$} & $8.69$ \scalebox{0.5}{$\pm 0.39$} & $116.87$ \scalebox{0.5}{$\pm 0.16$} &
$108.65$ \scalebox{0.5}{$\pm 0.11$} & $21.52$ \scalebox{0.5}{$\pm 0.13$} & $130.17$ \scalebox{0.5}{$\pm 0.09$}
\\
DIWAE $(\alpha=0.5)$&
$\bf85.78$ \scalebox{0.5}{$\pm 0.02$} & $4.47$ \scalebox{0.5}{$\pm 0.02$} & $90.25$ \scalebox{0.5}{$\pm 0.03$} &
$\bf107.01$ \scalebox{0.5}{$\pm 0.11$} & $8.64$ \scalebox{0.5}{$\pm 0.07$} & $\bf115.66$ \scalebox{0.5}{$\pm 0.17$} &
$\bf107.34$ \scalebox{0.5}{$\pm 0.17$} & $17.61$ \scalebox{0.5}{$\pm 0.18$} & $124.96$ \scalebox{0.5}{$\pm 0.14$}
\\
DIWAE $(\alpha=1.0)$& 
$\bf85.78$ \scalebox{0.5}{$\pm 0.03$} & $\bf4.21$ \scalebox{0.5}{$\pm 0.03$} & $\bf90.00$ \scalebox{0.5}{$\pm 0.06$} &
$107.47$ \scalebox{0.5}{$\pm 0.06$} & $\bf8.57$ \scalebox{0.5}{$\pm 0.14$} & $116.04$ \scalebox{0.5}{$\pm 0.18$} &
$107.54$ \scalebox{0.5}{$\pm 0.11$} & $\bf17.06$ \scalebox{0.5}{$\pm 0.35$} & $\bf124.60$ \scalebox{0.5}{$\pm 0.29$}
\\
WNI-IWAE &
$85.81$ \scalebox{0.5}{$\pm 0.01$} & $4.33$ \scalebox{0.5}{$\pm 0.03$} & $90.14$ \scalebox{0.5}{$\pm 0.04$} &
$107.15$ \scalebox{0.5}{$\pm 0.08$} & $8.78$ \scalebox{0.5}{$\pm 0.17$} & $115.93$ \scalebox{0.5}{$\pm 0.10$} &
$107.98$ \scalebox{0.5}{$\pm 0.19$} & $22.18$ \scalebox{0.5}{$\pm 0.33$} & $130.16$ \scalebox{0.5}{$\pm 0.14$}
\\
\midrule
\end{tabular}

\caption{Test set evaluation of the four models when trained with $64$ importance samples. $\Delta_\infer = \ln p_\theta(x) - \L_{64}(x)$.} 
\label{table:iwae64}
\begin{tabular}{l|ccc|ccc|ccc}
\toprule
& \multicolumn{3}{c|}{MNIST} & \multicolumn{3}{c|}{OMNIGLOT} & \multicolumn{3}{c}{CALTECH} \\
& $-\ln p_\theta(x)$ & $\Delta_\infer$ & $-\L_{64}(x)$
& $-\ln p_\theta(x)$ & $\Delta_\infer$ & $-\L_{64}(x)$
& $-\ln p_\theta(x)$ & $\Delta_\infer$ & $-\L_{64}(x)$\\
\midrule
IWAE & 
$86.06$ \scalebox{0.5}{$\pm 0.03$} & $4.41$ \scalebox{0.5}{$\pm 0.10$} & $90.48$ \scalebox{0.5}{$\pm 0.07$} &
$107.31$ \scalebox{0.5}{$\pm 0.14$} & $\bf6.66$ \scalebox{0.5}{$\pm 0.22$} & $113.97$ \scalebox{0.5}{$\pm 0.10$} &
$108.89$ \scalebox{0.5}{$\pm 0.35$} & $16.51$ \scalebox{0.5}{$\pm 0.32$} & $125.40$ \scalebox{0.5}{$\pm 0.25$}
\\
DIWAE $(\alpha=0.5)$ &
$\bf85.55$ \scalebox{0.5}{$\pm 0.02$} & $\bf3.01$ \scalebox{0.5}{$\pm 0.01$} & $\bf88.56$ \scalebox{0.5}{$\pm 0.02$} &
$\bf106.02$ \scalebox{0.5}{$\pm 0.01$} & $6.98$ \scalebox{0.5}{$\pm 0.06$} & $113.00$ \scalebox{0.5}{$\pm 0.07$} &
$\bf106.94$ \scalebox{0.5}{$\pm 0.11$} & $\bf12.28$ \scalebox{0.5}{$\pm 0.14$} & $\bf119.22$ \scalebox{0.5}{$\pm 0.11$}
\\
DIWAE $(\alpha=1.0)$& 
$\bf85.55$ \scalebox{0.5}{$\pm 0.02$} & $3.15$ \scalebox{0.5}{$\pm 0.02$} & $88.70$ \scalebox{0.5}{$\pm 0.04$} &
$106.15$ \scalebox{0.5}{$\pm 0.03$} & $6.70$ \scalebox{0.5}{$\pm 0.05$} & $\bf112.85$ \scalebox{0.5}{$\pm 0.07$} &
$106.96$ \scalebox{0.5}{$\pm 0.11$} & $12.94$ \scalebox{0.5}{$\pm 0.22$} & $119.87$ \scalebox{0.5}{$\pm 0.16$}
\\
WNI-IWAE &
$85.64$ \scalebox{0.5}{$\pm 0.03$} & $3.10$ \scalebox{0.5}{$\pm 0.01$} & $88.74$ \scalebox{0.5}{$\pm 0.03$} &
$106.17$ \scalebox{0.5}{$\pm 0.07$} & $7.11$ \scalebox{0.5}{$\pm 0.07$} & $113.28$ \scalebox{0.5}{$\pm 0.13$} &
$108.15$ \scalebox{0.5}{$\pm 0.11$} & $14.42$ \scalebox{0.5}{$\pm 0.20$} & $122.57$ \scalebox{0.5}{$\pm 0.10$}
\\
\midrule
\end{tabular}
\end{table}

\cref{table:iwae8,table:iwae64} extends the model evaluation to IWAE-$8$ and IWAE-$64$. We see that the denoising IWAE (DIWAE) and weight-normalized inference IWAE (WNI-IWAE) consistently out-perform the standard IWAE on test set log-likelihood evaluations. Furthermore, the regularized models frequently reduced the inference gap as well. Our results demonstrate that AIR is a highly effective regularizer even when a large number of importance samples are used.

Our main experimental contribution in this section is the verification that increasing the number of importance samples results in less underfitting when the inference model is over-regularized. In contrast to $k=1$, where aggressively increasing the regularization strength can cause considerable underfitting, \Cref{fig:strength_ablation} shows that increasing the number of importance samples to $k=8$ and $k=64$ makes the models much more robust to mis-specified choices of regularization strength. Interestingly, we also observed that the optimal regularization strength (determined using the validation set) increases with $k$ (see \cref{table:strengths} for details). The robustness of importance sampling when paired with amortized inference regularization makes AIR an effective and practical way to regularize IWAE.

\begin{figure}[h]
\centering
\includegraphics[width=.9\textwidth]{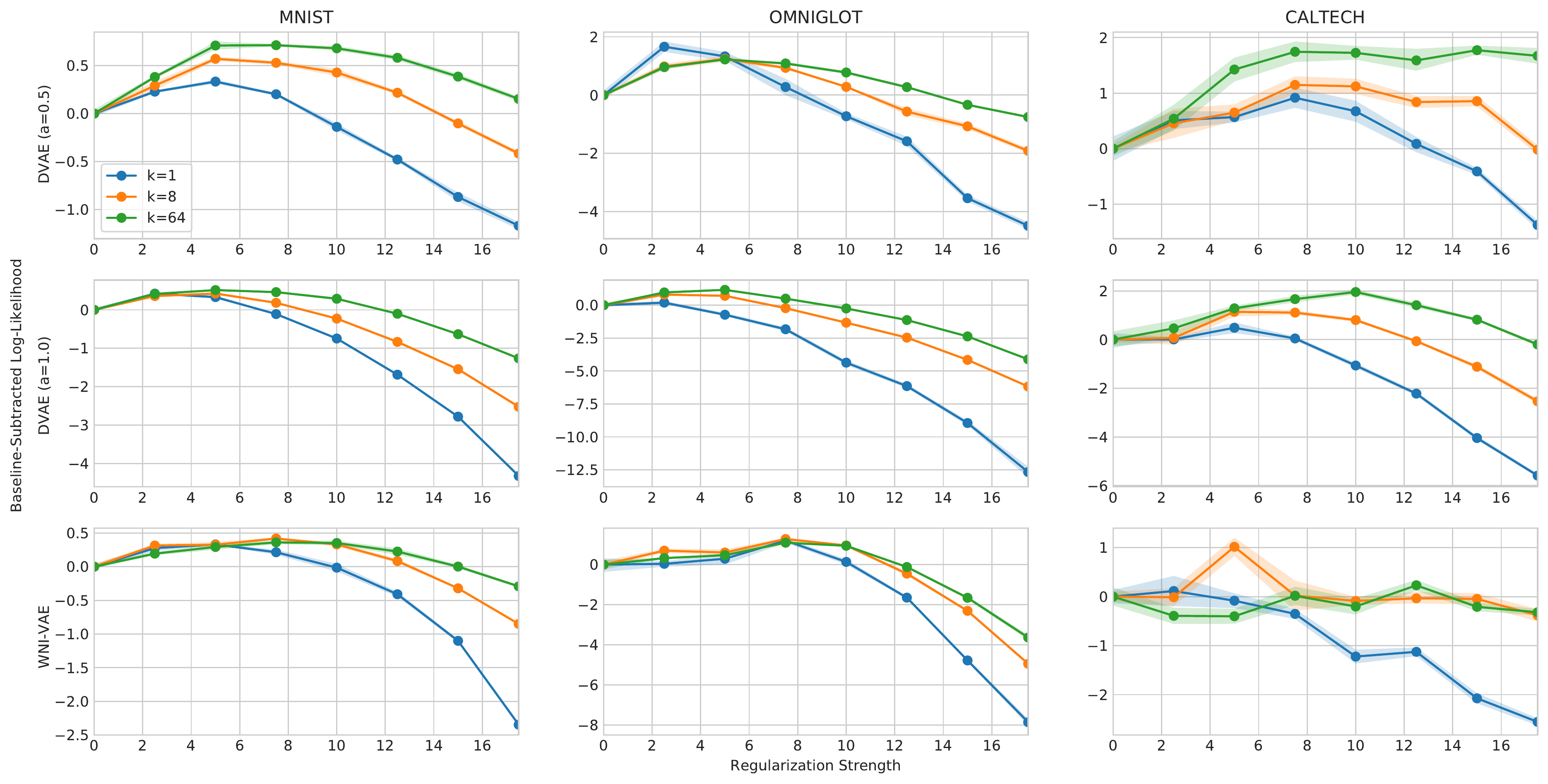}
\caption{Evaluation of the log-likelihood performance of all three proposed models as we vary the regularization parameter (see \cref{table:strengths} for definition) and number of importance samples $k$. To compare across different $k$'s, the performance without regularization (IWAE-$k$ baseline) is subtracted. We see that IWAE-$64$ is the least likely to underfit when the regularization parameter value is high.}
\label{fig:strength_ablation}
\end{figure}

\subsection{Are High Signal-to-Noise Ratio Gradients Necessarily Better?}

We note the existence of a related work \cite{rainforth2018tighter} that also concluded that approximating maximum likelihood training is not necessarily better. However, \cite{rainforth2018tighter} focused on increasing the signal-to-noise ratio of the gradient updates and analyzed the trade-off between importance sampling and Monte Carlo sampling under budgetary constraints. 
An in-depth discussion of these two works within the context of generalization is provided in \cref{app:snr}.

\section{Conclusion}

In this paper, we challenged the conventional role that amortized inference plays in training deep generative models. In addition to expediting variational inference, amortized inference introduces new ways to regularize maximum likelihood training. We considered a special case of amortized inference regularization (AIR) where the inference model must learn a smoothed mapping from $\X \to \Q$ and showed that the denoising variational autoencoder (DVAE) and weight-normalized inference (WNI) are effective instantiations of AIR. Promising directions for future work include replacing denoising with adversarial training \cite{goodfellow2014explaining} and weight normalization with spectral normalization \cite{miyato2018spectral}. Furthermore, we demonstrated that AIR plays a crucial role in the regularization of IWAE, and that higher levels of regularization may be necessary due to the attenuating effects of importance sampling on AIR. We believe that variational family expansion by Monte Carlo methods \cite{hoffman2017learning} may exhibit the same attenuating effect on AIR and recommend this as an additional research direction.

\subsection*{Acknowledgements}
This research was supported by TRI, NSF (\#1651565, \#1522054,
\#1733686 ), ONR, Sony, and FLI. Toyota Research Institute provided funds to assist the authors with their research but this article solely reflects the opinions and conclusions of its authors and not TRI or any other Toyota entity.


\bibliography{main}

\begin{thebibliography}{10}

\bibitem{kingma2013auto}
Diederik~P Kingma and Max Welling.
\newblock {A}uto-{E}ncoding {V}ariational {B}ayes.
\newblock {\em arXiv preprint arXiv:1312.6114}, 2013.

\bibitem{kingma2014semi}
Diederik~P Kingma, Shakir Mohamed, Danilo~Jimenez Rezende, and Max Welling.
\newblock {S}emi-{S}upervised {L}earning {W}ith {D}eep {G}enerative {M}odels.
\newblock In {\em Advances In Neural Information Processing Systems}, pages
  3581--3589, 2014.

\bibitem{kim2018disentangling}
Hyunjik Kim and Andriy Mnih.
\newblock {D}isentangling {B}y {F}actorising.
\newblock {\em arXiv preprint arXiv:1802.05983}, 2018.

\bibitem{chen2018isolating}
Tian~Qi Chen, Xuechen Li, Roger Grosse, and David Duvenaud.
\newblock {I}solating {S}ources {O}f {D}isentanglement {I}n {V}ariational
  {A}utoencoders.
\newblock {\em arXiv preprint arXiv:1802.04942}, 2018.

\bibitem{kim2018semi}
Yoon Kim, Sam Wiseman, Andrew~C Miller, David Sontag, and Alexander~M Rush.
\newblock {S}emi-{A}mortized {V}ariational {A}utoencoders.
\newblock {\em arXiv preprint arXiv:1802.02550}, 2018.

\bibitem{kingma2016improved}
Diederik~P Kingma, Tim Salimans, Rafal Jozefowicz, Xi~Chen, Ilya Sutskever, and
  Max Welling.
\newblock {I}mproved {V}ariational {I}nference {W}ith {I}nverse
  {A}utoregressive {F}low.
\newblock In {\em Advances In Neural Information Processing Systems}, pages
  4743--4751, 2016.

\bibitem{sonderby2016ladder}
Casper~Kaae S{\o}nderby, Tapani Raiko, Lars Maal{\o}e, S{\o}ren~Kaae
  S{\o}nderby, and Ole Winther.
\newblock {L}adder {V}ariational {A}utoencoders.
\newblock In {\em Advances In Neural Information Processing Systems}, pages
  3738--3746, 2016.

\bibitem{burda2015importance}
Yuri Burda, Roger Grosse, and Ruslan Salakhutdinov.
\newblock {I}mportance {W}eighted {A}utoencoders.
\newblock {\em arXiv preprint arXiv:1509.00519}, 2015.

\bibitem{rezende2014stochastic}
Danilo~Jimenez Rezende, Shakir Mohamed, and Daan Wierstra.
\newblock {S}tochastic {B}ackpropagation {A}nd {A}pproximate {I}nference {I}n
  {D}eep {G}enerative {M}odels.
\newblock {\em arXiv preprint arXiv:1401.4082}, 2014.

\bibitem{cremer2018inference}
Chris Cremer, Xuechen Li, and David Duvenaud.
\newblock {I}nference {S}uboptimality {I}n {V}ariational {A}utoencoders.
\newblock {\em arXiv preprint arXiv:1801.03558}, 2018.

\bibitem{wu2016quantitative}
Yuhuai Wu, Yuri Burda, Ruslan Salakhutdinov, and Roger Grosse.
\newblock {O}n {T}he {Q}uantitative {A}nalysis {O}f {D}ecoder-{B}ased
  {G}enerative {M}odels.
\newblock {\em arXiv preprint arXiv:1611.04273}, 2016.

\bibitem{krishnan2017challenges}
Rahul~G Krishnan, Dawen Liang, and Matthew Hoffman.
\newblock On the challenges of learning with inference networks on sparse,
  high-dimensional data.
\newblock {\em arXiv preprint arXiv:1710.06085}, 2017.

\bibitem{maaloe2016auxiliary}
Lars Maal{\o}e, Casper~Kaae S{\o}nderby, S{\o}ren~Kaae S{\o}nderby, and Ole
  Winther.
\newblock {A}uxiliary {D}eep {G}enerative {M}odels.
\newblock {\em arXiv preprint arXiv:1602.05473}, 2016.

\bibitem{ranganath2016hierarchical}
Rajesh Ranganath, Dustin Tran, and David Blei.
\newblock {H}ierarchical {V}ariational {M}odels.
\newblock In {\em International Conference On Machine Learning}, pages
  324--333, 2016.

\bibitem{ganchev2010posterior}
Kuzman Ganchev, Jennifer Gillenwater, Ben Taskar, et~al.
\newblock {P}osterior {R}egularization {F}or {S}tructured {L}atent {V}ariable
  {M}odels.
\newblock {\em Journal of Machine Learning Research}, 11(Jul):2001--2049, 2010.

\bibitem{zhu2014bayesian}
Jun Zhu, Ning Chen, and Eric~P Xing.
\newblock {B}ayesian {I}nference {W}ith {P}osterior {R}egularization {A}nd
  {A}pplications {T}o {I}nfinite {L}atent {S}vms.
\newblock {\em The Journal of Machine Learning Research}, 15(1):1799--1847,
  2014.

\bibitem{im2017denoising}
Daniel~Jiwoong Im, Sungjin Ahn, Roland Memisevic, Yoshua Bengio, et~al.
\newblock {D}enoising {C}riterion {F}or {V}ariational {A}uto-{E}ncoding
  {F}ramework.
\newblock In {\em AAAI}, pages 2059--2065, 2017.

\bibitem{salimans2016weight}
Tim Salimans and Diederik~P Kingma.
\newblock {W}eight {N}ormalization: {A} {S}imple {R}eparameterization {T}o
  {A}ccelerate {T}raining {O}f {D}eep {N}eural {N}etworks.
\newblock In {\em Advances In Neural Information Processing Systems}, pages
  901--909, 2016.

\bibitem{kingma2014adam}
Diederik~P Kingma and Jimmy Ba.
\newblock {A}dam: {A} method {F}or {S}tochastic {O}ptimization.
\newblock {\em arXiv preprint arXiv:1412.6980}, 2014.

\bibitem{hoffman2013stochastic}
Matthew~D Hoffman, David~M Blei, Chong Wang, and John Paisley.
\newblock {S}tochastic {V}ariational {I}nference.
\newblock {\em The Journal of Machine Learning Research}, 14(1):1303--1347,
  2013.

\bibitem{bousquet2017optimal}
Olivier Bousquet, Sylvain Gelly, Ilya Tolstikhin, Carl-Johann Simon-Gabriel,
  and Bernhard Schoelkopf.
\newblock {F}rom {O}ptimal {T}ransport {T}o {G}enerative {M}odeling: {T}he
  {VEGAN} {C}ookbook.
\newblock {\em arXiv preprint arXiv:1705.07642}, 2017.

\bibitem{cremer2017reinterpreting}
Chris Cremer, Quaid Morris, and David Duvenaud.
\newblock {R}einterpreting {I}mportance-{W}eighted {A}utoencoders.
\newblock {\em arXiv preprint arXiv:1704.02916}, 2017.

\bibitem{rainforth2018tighter}
Tom Rainforth, Adam~R Kosiorek, Tuan~Anh Le, Chris~J Maddison, Maximilian Igl,
  Frank Wood, and Yee~Whye Teh.
\newblock {T}ighter {V}ariational {B}ounds {A}re {N}ot {N}ecessarily {B}etter.
\newblock {\em arXiv preprint arXiv:1802.04537}, 2018.

\bibitem{goodfellow2014explaining}
Ian~J Goodfellow, Jonathon Shlens, and Christian Szegedy.
\newblock {E}xplaining {A}nd {H}arnessing {A}dversarial {E}xamples.
\newblock {\em arXiv preprint arXiv:1412.6572}, 2014.

\bibitem{miyato2018spectral}
Takeru Miyato, Toshiki Kataoka, Masanori Koyama, and Yuichi Yoshida.
\newblock {S}pectral {N}ormalization {F}or {G}enerative {A}dversarial
  {N}etworks.
\newblock {\em arXiv preprint arXiv:1802.05957}, 2018.

\bibitem{hoffman2017learning}
Matthew~D Hoffman.
\newblock {L}earning {D}eep {L}atent {G}aussian {M}odels {W}ith {M}arkov
  {C}hain {M}onte {C}arlo.
\newblock In {\em International Conference On Machine Learning}, pages
  1510--1519, 2017.

\bibitem{li2016renyi}
Yingzhen Li and Richard~E Turner.
\newblock {R}{\'e}nyi {D}ivergence {V}ariational {I}nference.
\newblock In {\em Advances In Neural Information Processing Systems}, pages
  1073--1081, 2016.

\bibitem{tomczak2017vae}
Jakub~M Tomczak and Max Welling.
\newblock {VAE} {W}ith {A} {V}ampprior.
\newblock {\em arXiv preprint arXiv:1705.07120}, 2017.

\bibitem{smith2018sgd}
Samuel~L. Smith and Quoc~V. Le.
\newblock {A} bayesian {P}erspective {O}n {G}eneralization {A}nd {S}tochastic
  {G}radient {D}escent.
\newblock In {\em International Conference On Learning Representations}, 2018.

\bibitem{dinh2017sharp}
Laurent Dinh, Razvan Pascanu, Samy Bengio, and Yoshua Bengio.
\newblock {S}harp {M}inima {C}an {G}eneralize {F}or {D}eep {N}ets.
\newblock {\em arXiv preprint arXiv:1703.04933}, 2017.

\bibitem{masters2018revisiting}
Dominic Masters and Carlo Luschi.
\newblock {R}evisiting {S}mall {B}atch {T}raining {F}or {D}eep {N}eural
  {N}etworks.
\newblock {\em arXiv preprint arXiv:1804.07612}, 2018.

\bibitem{zhang2016understanding}
Chiyuan Zhang, Samy Bengio, Moritz Hardt, Benjamin Recht, and Oriol Vinyals.
\newblock {U}nderstanding {D}eep {L}earning {R}equires {R}ethinking
  {G}eneralization.
\newblock {\em arXiv preprint arXiv:1611.03530}, 2016.

\bibitem{banerjee2005clustering}
Arindam Banerjee, Srujana Merugu, Inderjit~S Dhillon, and Joydeep Ghosh.
\newblock {C}lustering {W}ith {B}regman {D}ivergences.
\newblock {\em Journal of machine learning research}, 6(Oct):1705--1749, 2005.

\end{thebibliography}
\bibliographystyle{unsrt}

\newpage
\appendix

\section{Overly Expressive Amortization Family Hurts Generalization}
\label{app:large_encoder}
In the experiments by \cite{cremer2018inference}, they observed that an overly expressive amortization family increases the test set inference gap, but does not impact the test set log-likelihood. We show in \cref{table:large_encoder} that \cite{cremer2018inference}'s observation is not true in general, and that an overly expressive amortization family can in fact hurt test set log-likelihood. Details regarding the architectures are provided in \cref{app:architecture}.

\begin{table}[!h]
\setlength\tabcolsep{3pt} 
\tiny
\centering
\caption{Performance evaluation when an over-expressive amortization family is used (i.e. a larger encoder). Comparison is made against models that use a smaller encoder. The results show that using a large encoder consistently hurts generalization by over $1$ nat.} 
\label{table:large_encoder}
\begin{tabular}{l|ccc|ccc|ccc}
\toprule
& \multicolumn{3}{c|}{MNIST ($k = 1$)} & \multicolumn{3}{c|}{MNIST ($k = 8$)} & \multicolumn{3}{c}{MNIST ($k = 64$)} \\
& $-\ln p_\theta(x)$ & $\Delta_\infer$ & $-\L_1(x)$
& $-\ln p_\theta(x)$ & $\Delta_\infer$ & $-\L_8(x)$
& $-\ln p_\theta(x)$ & $\Delta_\infer$ & $-\L_64(x)$\\
\midrule
IWAE (Large Encoder) &
$87.43$ \scalebox{0.5}{$\pm 0.05$} & $11.32$ \scalebox{0.5}{$\pm 0.21$} & $98.74$ \scalebox{0.5}{$\pm 0.25$} &
$86.98$ \scalebox{0.5}{$\pm 0.07$} & $8.00$ \scalebox{0.5}{$\pm 0.18$} & $94.98$ \scalebox{0.5}{$\pm 0.17$} &
$86.70$ \scalebox{0.5}{$\pm 0.06$} & $5.91$ \scalebox{0.5}{$\pm 0.11$} & $92.61$ \scalebox{0.5}{$\pm 0.10$} 
\\
IWAE &
$86.93$ \scalebox{0.5}{$\pm 0.04$} & $8.54$ \scalebox{0.5}{$\pm 0.14$} & $95.48$ \scalebox{0.5}{$\pm 0.07$} & 
$86.21$ \scalebox{0.5}{$\pm 0.01$} & $6.13$ \scalebox{0.5}{$\pm 0.03$} & $92.34$ \scalebox{0.5}{$\pm 0.02$} &
$86.06$ \scalebox{0.5}{$\pm 0.03$} & $4.41$ \scalebox{0.5}{$\pm 0.10$} & $90.48$ \scalebox{0.5}{$\pm 0.07$} 
\\
DIWAE $(\alpha=0.5)$&
$86.46$ \scalebox{0.5}{$\pm 0.02$} & $\bf6.34$ \scalebox{0.5}{$\pm 0.05$} & $\bf92.80$ \scalebox{0.5}{$\pm 0.07$} & 
$\bf85.78$ \scalebox{0.5}{$\pm 0.02$} & $4.47$ \scalebox{0.5}{$\pm 0.02$} & $90.25$ \scalebox{0.5}{$\pm 0.03$} &
$\bf85.55$ \scalebox{0.5}{$\pm 0.02$} & $\bf3.01$ \scalebox{0.5}{$\pm 0.01$} & $\bf88.56$ \scalebox{0.5}{$\pm 0.02$} 
\\
DIWAE $(\alpha=1.0)$& 
$86.51$ \scalebox{0.5}{$\pm 0.02$} & $6.83$ \scalebox{0.5}{$\pm 0.04$} & $93.35$ \scalebox{0.5}{$\pm 0.06$} & 
$\bf85.78$ \scalebox{0.5}{$\pm 0.03$} & $\bf4.21$ \scalebox{0.5}{$\pm 0.03$} & $\bf90.00$ \scalebox{0.5}{$\pm 0.06$} &
$\bf85.55$ \scalebox{0.5}{$\pm 0.02$} & $3.15$ \scalebox{0.5}{$\pm 0.02$} & $88.70$ \scalebox{0.5}{$\pm 0.04$} 
\\
WNI-IWAE &
$\bf86.42$ \scalebox{0.5}{$\pm 0.01$} & $6.68$ \scalebox{0.5}{$\pm 0.01$} & $93.10$ \scalebox{0.5}{$\pm 0.02$} & 
$85.81$ \scalebox{0.5}{$\pm 0.01$} & $4.33$ \scalebox{0.5}{$\pm 0.03$} & $90.14$ \scalebox{0.5}{$\pm 0.04$} &
$85.64$ \scalebox{0.5}{$\pm 0.03$} & $3.10$ \scalebox{0.5}{$\pm 0.01$} & $88.74$ \scalebox{0.5}{$\pm 0.03$} 
\\
\midrule
\end{tabular}

\end{table}


\section{Revisiting \cite{im2017denoising}'s Denoising Variational Autoencoder Analysis}\label{app:dvae}

In \cite{im2017denoising}'s Lemma 1, they considered a joint distribution $p_\theta(x, z)$. They introduced an auxiliary variable $z'$ into their inference model (here $z'$ takes on the role of the perturbed input $\tilde{x} = x + \veps$. To avoid confusion, we stick to the notation used in their Lemma) and considered the inference model
\begin{align}
q_\varphi(z \giv z')q_\psi(z' \giv x).
\end{align}
They considered two ways to use this inference model. The first approach is to marginalize the auxiliary latent variable $z'$. This defines the resulting inference model
\begin{align}
q_\phi(z \giv x) = \int q_\varphi(z \giv z')q_\psi(z' \giv x) dz'.
\end{align}
This yields the lower bound
\begin{align}
\L_a = \Expect_{q_\phi(z \giv x)} \brac{\ln \frac{p_\theta(x, z)}{q_\phi(z \giv x)}}.
\end{align}
Next, they considered an alternative lower bound
\begin{align}
\L_b = \Expect_{q_\varphi(z \giv z')q_\psi(z' \giv x)}\brac{\ln \frac{p_\theta(x, z)}{q_\varphi(z \giv z')}}.
\end{align}
\cite{im2017denoising}'s Lemma 1 claims that 
\begin{enumerate}
\item $\L_a$ and $\L_b$ are valid lower bounds of $\ln p_\theta(x)$
\item $\L_b \ge \L_a$.
\end{enumerate}
Using Lemma 1, \cite{im2017denoising} motivated the denoising variational autoencoder by concluding that it provides a tighter bound than marginalization of the noise variable. Although statement 1 is correct, statement 2 is not. Their proof of statement 2 is presented as follows
\begin{align}
\Expect_{q_\varphi(z \giv z')q_\psi(z' \giv x)}\brac{\ln \frac{p_\theta(x, z)}{q_\varphi(z \giv z')}}
&\stackrel{?}{=} \Expect_{q_\phi(z \giv x)}\brac{\ln \frac{p_\theta(x, z)}{q_\varphi(z \giv z')}} \label{eq:step}\\
&= \Expect_{q_\phi(z \giv x)} \brac{\log p_\theta(x, z)} - \Expect_{q_\phi(z \giv x)} \brac{\ln q_\phi(z \giv z')} \\
&\ge \Expect_{q_\phi(z \giv x)} \brac{\log p_\theta(x, z)} - \Expect_{q_\phi(z \giv x)} \brac{\ln q_\phi(z \giv z')} \\
&= \Expect_{q_\phi(z \giv x)} \brac{\ln \frac{p_\theta(x, z)}{q_\phi(z \giv x)}}
\end{align}
We indicate the mistake with $\stackrel{?}{=}$; their proof of statement 2 relied on the assumption that
\begin{align}
\Expect_{q_\varphi(z \giv z')q_\psi(z' \giv x)}\brac{\ln \frac{p_\theta(x, z)}{q_\varphi(z \giv z')}} = \Expect_{q_\phi(z \giv x)}\brac{\ln \frac{p_\theta(x, z)}{q_\varphi(z \giv z')}}.
\end{align}
Crucially, the RHS is ill-defined since it does not take the expectation over $z'$, whereas the LHS explicitly specifies an expectation over $z' \sim q_\psi(z' \giv x)$. This difference, while subtle, invalidates the subsequent steps. If we fix \cref{eq:step} and attempt to see if the rest of the proof still follows, we will find that
\begin{align}
\Expect_{q_\varphi(z \giv z')q_\psi(z' \giv x)}\brac{\ln \frac{p_\theta(x, z)}{q_\varphi(z \giv z')}}
&= \Expect_{q_\phi(z \giv x)} \brac{\log p_\theta(x, z)} 
- \Expect_{q_\varphi(z \giv z')q_\psi(z' \giv x)} \brac{\ln q_\psi(z \giv z')} \\
&\not\ge \Expect_{q_\phi(z \giv x)} \brac{\log p_\theta(x, z)} 
- \Expect_{q_\varphi(z \giv z')q_\psi(z' \giv x)} \brac{\ln q_\phi(z \giv x)}\\
&= \Expect_{q_\phi(z \giv x)} \brac{\ln \frac{p_\theta(x, z)}{q_\phi(z \giv x)}}.
\end{align}
Indeed, the inequality will point the other way since
\begin{align}
\Expect_{q_\varphi(z \giv z')q_\psi(z' \giv x)} \brac{\ln q_\psi(z \giv z') - \ln q_\phi(z \giv x)}
&= \Expect_{q_\psi(z' \giv x)}\Expect_{q_\varphi(z \giv z')} \ln \frac{q_\varphi(z \giv z')}{q_\phi(z \giv x)} \\
&= \Expect_{q_\psi(z' \giv x)}\KL{q_\varphi(z \giv z')}{q_\phi(z \giv x)} \\
&\ge 0 \implies \\
-\Expect_{q_\varphi(z \giv z')q_\psi(z' \giv x)} \brac{\ln q_\psi(z \giv z')} &\le 
-\Expect_{q_\varphi(z \giv z')q_\psi(z' \giv x)} \brac{\ln q_\phi(z \giv x)}.
\end{align}
Their conclusion that marginalizing over the noise variable results in a looser bound is thus incorrect. In the text (beneath \cite{im2017denoising} Eq. (11)), they further implied that the denoising VAE and standard VAE objectives are not comparable. We show in \cref{prop:strength} that the denoising VAE objective is in fact a lower bound of the standard VAE objective.

\section{Importance-Weighted Stochastic Variational Inference}\label{app:iwsvi}

\begin{figure}[h]
\centering
\begin{subfigure}[b]{0.32\textwidth}
\includegraphics[width=\textwidth]{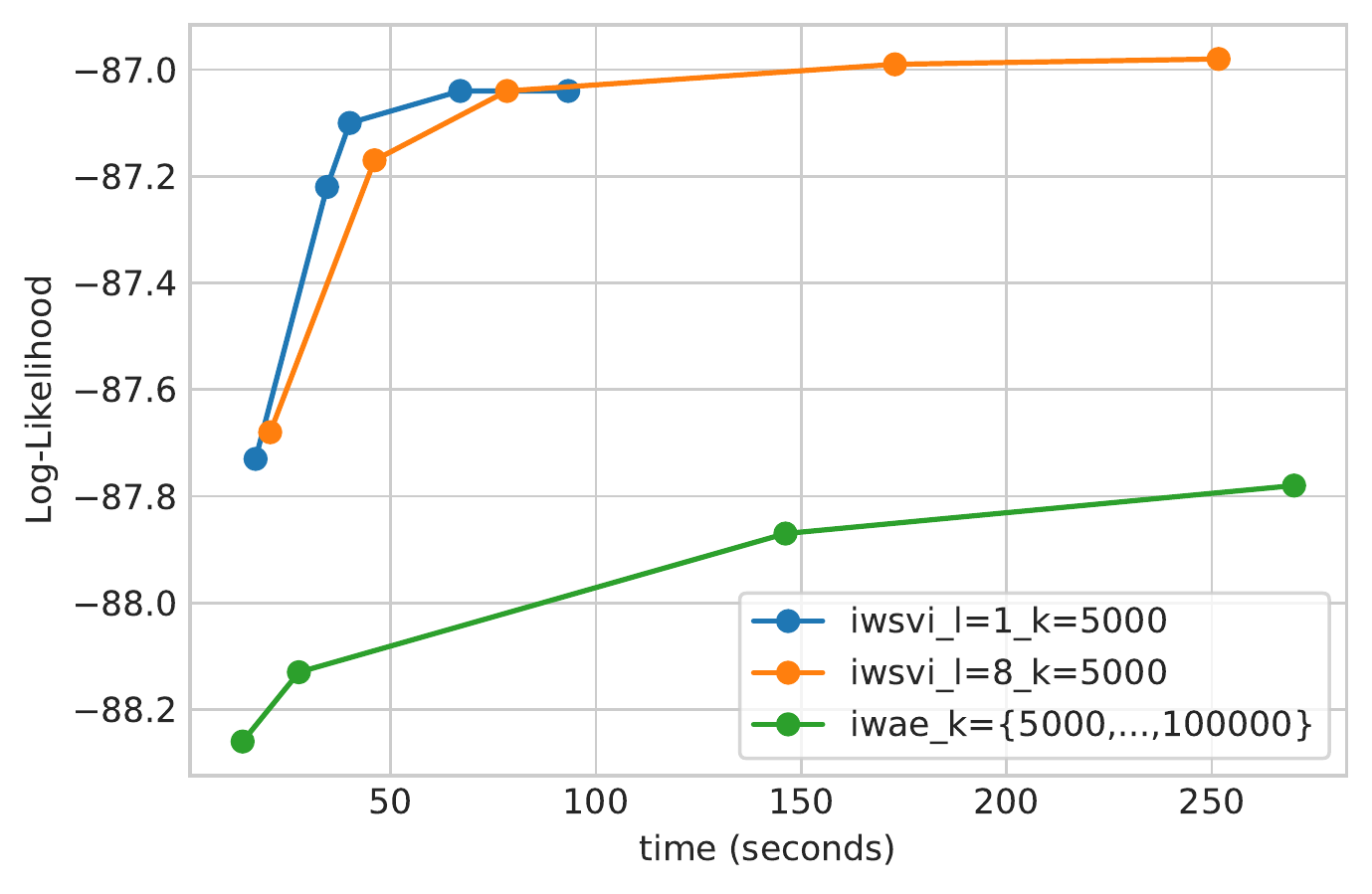}
\caption{}
\label{fig:a}
\end{subfigure}
\begin{subfigure}[b]{0.32\textwidth}
\includegraphics[width=\textwidth]{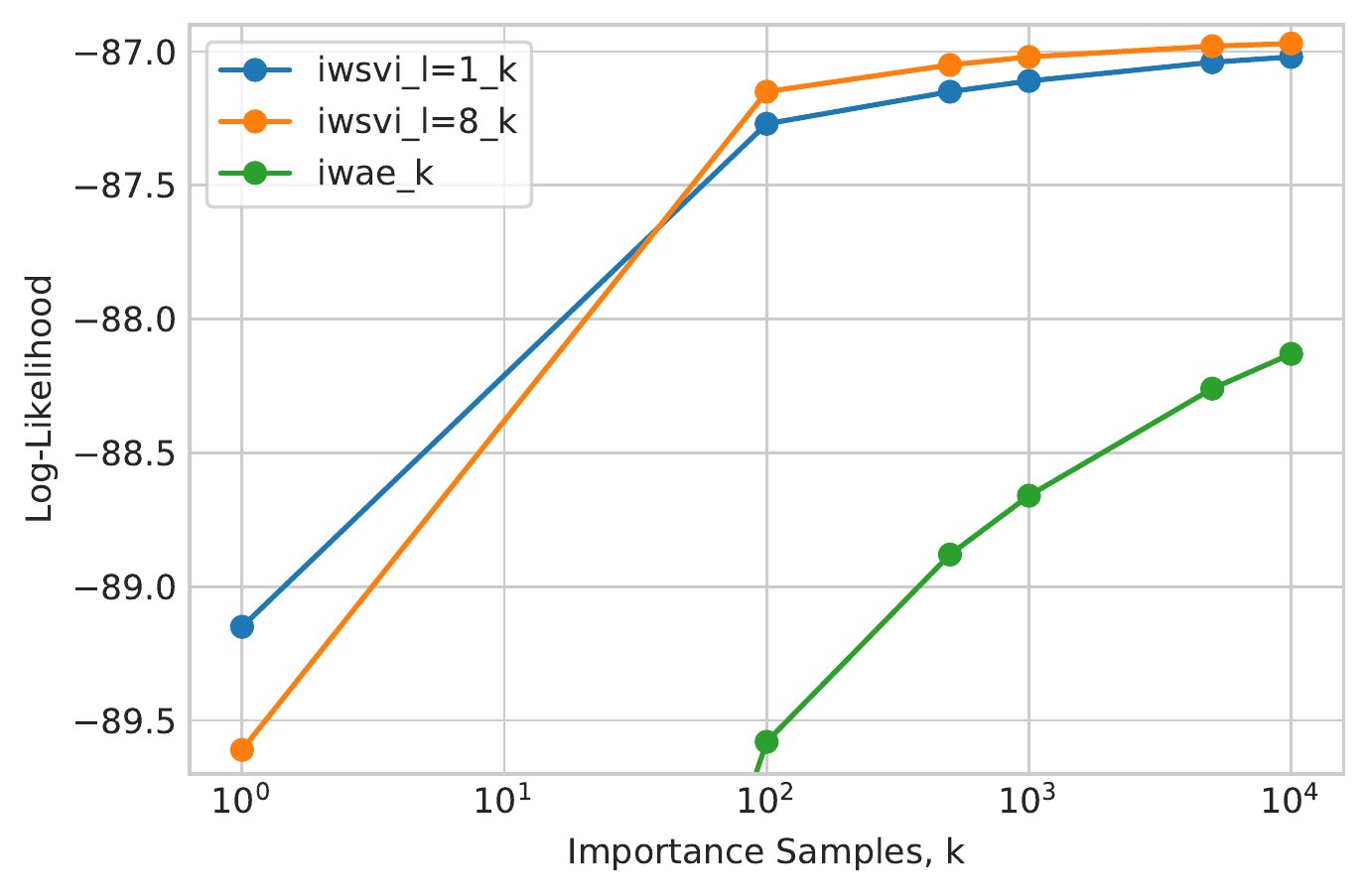}
\caption{}
\label{fig:b}
\end{subfigure}
\begin{subfigure}[b]{0.32\textwidth}
\includegraphics[width=\textwidth]{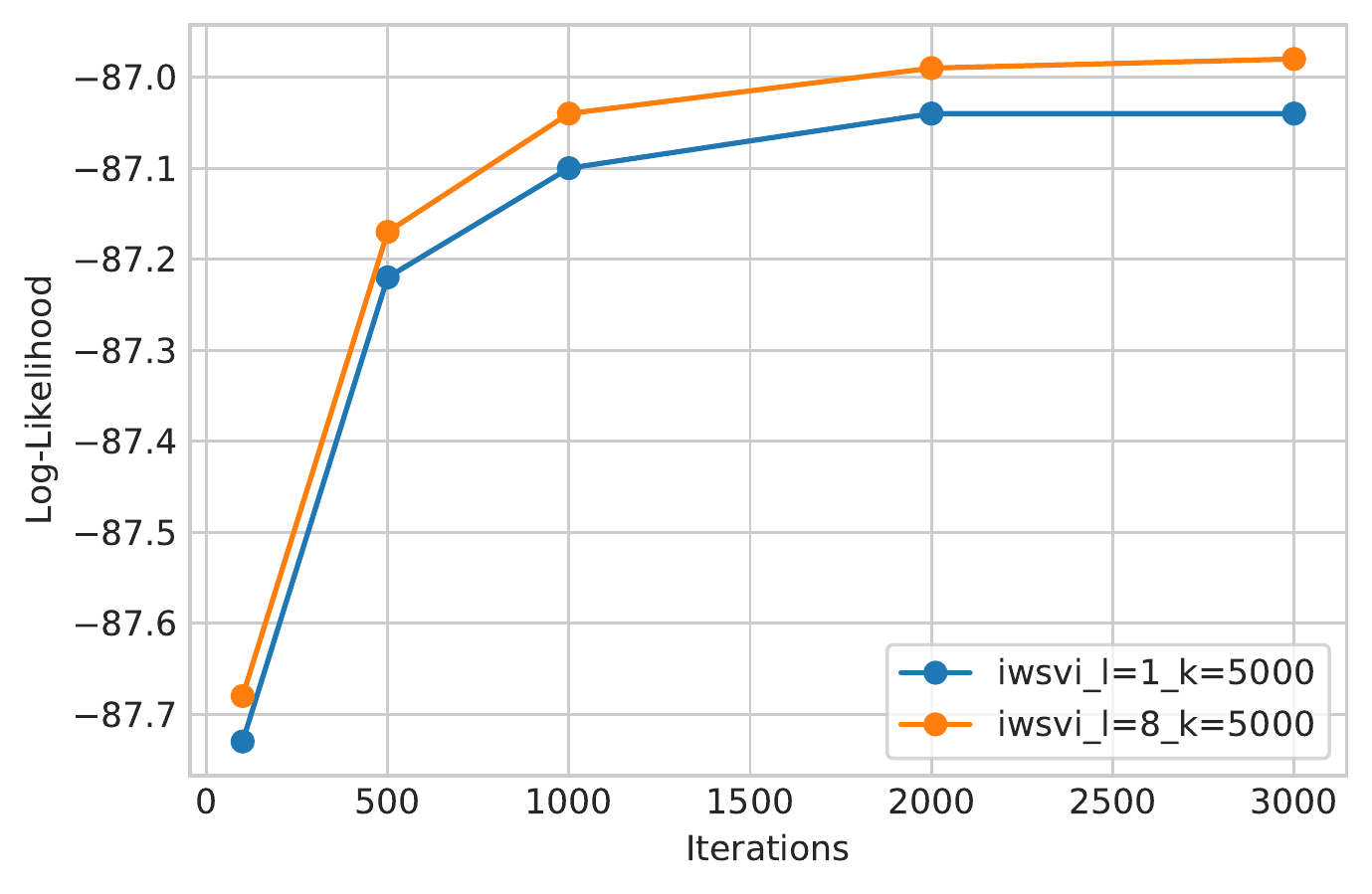}
\caption{}
\label{fig:c}
\end{subfigure}
\caption{Evaluation of IW-SVI versus IWAE-$k$ for a fixed generative model. IW-SVI out-performss IWAE-$k$ on both computation time and number of importance samples needed. Similar to \cite{wu2016quantitative}, we conclude that IWAE-$k$'s poor approximation of the log-likelihood is attributable to an overfit amortized inference model. \cref{fig:a}) IW-SVI computation time depends on the number of gradient update steps. IWAE-$k$ computation time depends on the number of importance samples $k$. IWAE-$100000$ still under-performs IW-SVI $(k=5000,\ell=1, T=100)$, demonstrating the efficacy of IW-SVI. \cref{fig:b}) Comparison of IWAE and IW-SVI $(T=3000)$ for different values of $k$. \cref{fig:c} Comparison of IW-SVI $(k=5000)$ for different values of $T$.}\label{fig:iwsvi}
\end{figure}

We propose a simple method to approximate the marginal $\ln p_\theta(x)$. A common approach for approximating the log marginal is the IWAE-$5000$ \cite{burda2015importance,sonderby2016ladder,li2016renyi,tomczak2017vae}, which proposes to compute $\L_{5000}(x\scolon \theta,\phi)$ where
\begin{align}
\ln p_\theta(x) \ge \L_k(x \scolon \theta, \phi) 
= \Expect_{z_1\ldots z_k \sim q_\phi(z \giv x)} \paren{\ln \frac{1}{k} \sum_{i=1}^k \frac{p_\theta(x, z^{(i)})}{q_\phi(z^{(i)} \giv x)}}.
\end{align}

However, this approach relies on the learned inference model $q_\phi(z \giv x)$, which might overfit to the training set. To address this issue, we propose to perform importance-weighted stochastic variational inference (IW-SVI)
\begin{align}
\ln p_\theta(x) \ge \L_k(x \scolon \theta, q^*_{x,\ell}) &= \Expect_{z_1\ldots z_k \sim q^*_{x,\ell}(z)} \paren{\ln \frac{1}{k} \sum_{i=1}^k \frac{p_\theta(z \giv x)}{q^*_{x,\ell}(z)}}, \\
\text{where } q^*_{x,\ell}  &= \argmax_{q \in \Q} \L_\ell(x \scolon \theta, q).\label{eq:optimize}
\end{align}
The optimization in \cref{eq:optimize} is approximate with $T$ gradient steps. As $k$ and $\ell$ increase, the approximation will approach the true log-likelihood. We approximate log-likelihood over the entire test set using $\Expect_{\hat{p}_\text{test}(x)}\L_k(p_\theta, q^*_{x,\ell} \scolon x)$. To reduce speed and memory cost during the per-sample optimization in \cref{eq:optimize},  we use a large $k = 5000$ but smaller $\ell = 8$, and approximately solved the optimization problem using $T=3000$ gradient steps. In comparison to IWAE-$5000$, we consistently observe significant improvement in the log-likelihood approximation. IW-SVI provides a simple alternative to Annealed Importance Sampling, requiring minimal modification to any existing IWAE-$k$ implementation.

\section{Are High Signal-to-Noise Ratio Gradients Necessarily Better?}\label{app:snr}

Our paper shares a similar high level message with a recent study by \cite{rainforth2018tighter}: that approximating maximum likelihood training is not necessarily better. However, we approach this message in very different ways. \cite{rainforth2018tighter} observed that importance sampling weakens the signal-to-noise ratio of the gradients used to update the amortized inference model. In response, they proposed to increase this ratio by increasing the number of Monte Carlo samples $m$ used to estimate the expectation in \cref{eq:vae}. Under a fixed budget of $T \ge mk$ (where $k$ is the number of importance samples and $m$ is the number of Monte Carlo samples), they observed that it may be desirable to trade off $k$ in order to increase $m$. Given an infinite budget, however, \cite{rainforth2018tighter}'s hypothesis would still conclude to increase $k$ as much as possible in order to approximate maximum likelihood training.

In contrast, we argue that it may be inherently desirable to regularize the maximum likelihood objective, and that amortized inference regularization is an effective means of doing so. From the perspective of generalization, it is also worth wondering whether high signal-to-noise ratio gradients are necessarily better. The desirability of noisy gradients for improving generalization is an active area of research \cite{smith2018sgd,dinh2017sharp,masters2018revisiting,zhang2016understanding}, and an extensive investigation of the role of gradient stochasticity in regularizing the amortized inference model is beyond the scope of our paper. To encourage future exploration in this direction, we show in \Cref{fig:snr_analysis} that the effect of gradient stochasticity is non-negligible. For the standard VAE, we observed that increasing $m$ can cause the model to overfit (on the amortized ELBO objective) over the course of training. Interestingly, we observed that DVAE does not experience this overfitting effect, suggesting that AIR is robust to larger values of $m$.

\begin{figure}[h]
\centering
\includegraphics[width=.9\textwidth]{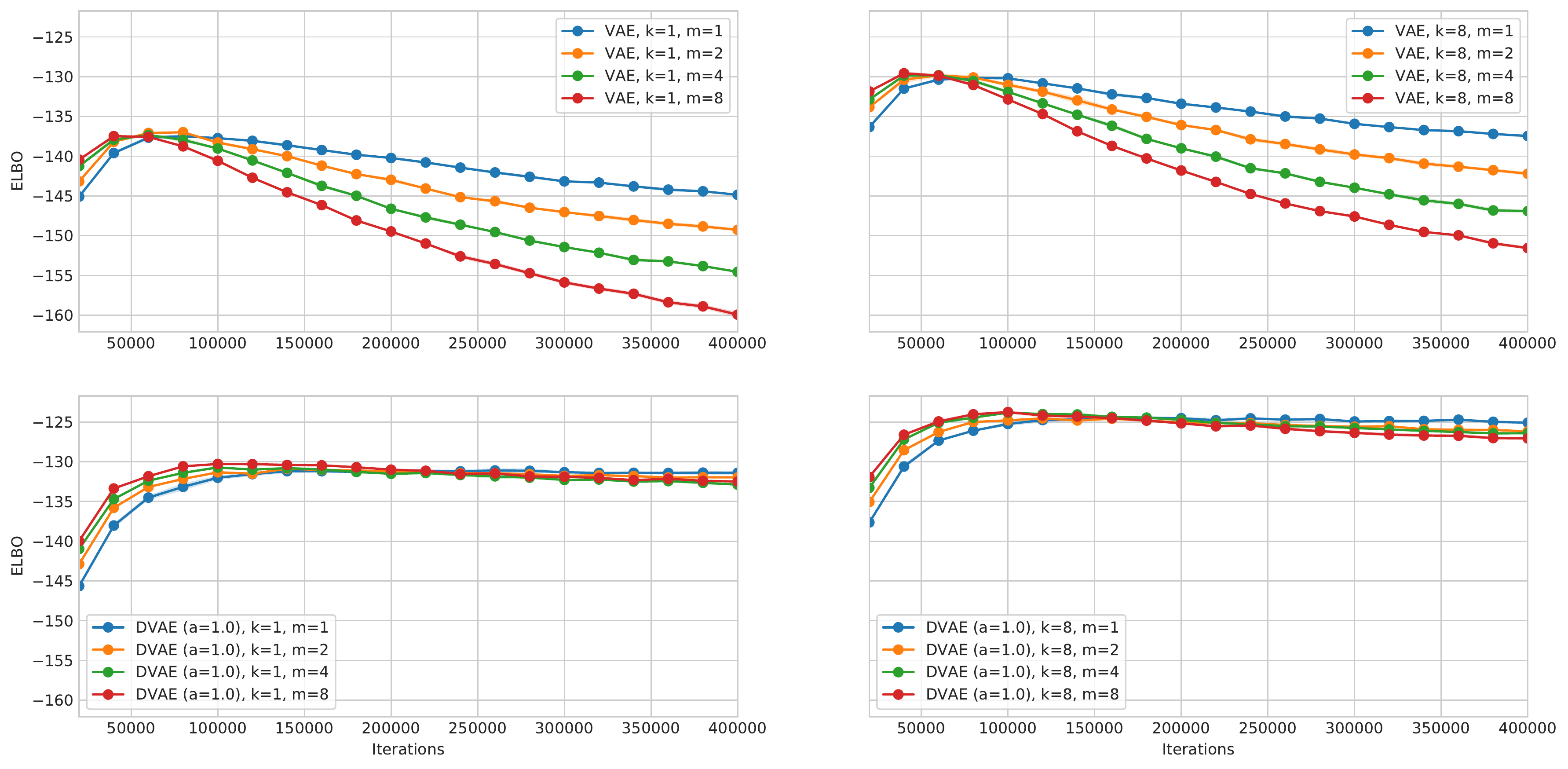}
\caption{Comparison of the test set amortized ELBO during training for VAE and DVAE as we vary the number of importance samples $k$ and the number of Monte Carlo samples $m$. In contrast to DVAE, VAE is susceptible to overfitting when $m$ is increased.}
\label{fig:snr_analysis}
\end{figure}

\clearpage
\section{Experimental Details}\label{app:architecture}

\textbf{Datasets.} We carried out experiments on the static MNIST, static OMNIGLOT, and Caltech 101 Silhouettes datasets. OMNIGLOT was statically binarized at the beginning of training via random sampling using the pixel real-values as Bernoulli parameters. Training, validation, and test split sizes are provided in \cref{table:datasets}. The MNIST validation set was created by randomly holding out $10000$ samples from the original $60000$-sample training set. The OMNIGLOT validation set was similarly created by randomly holding out $1345$ samples from the original $24345$-sample training set. 

\begin{table}[!h]
\setlength\tabcolsep{12pt} 
\scriptsize
\centering
\caption{Training, validation and test splits for each dataset.} \label{table:datasets}
\begin{tabular}{l|c|c|c}
\toprule
Dataset & 
Training Split &
Validation Split &
Test Split
\\
\midrule
MNIST &
$50000$ &
$10000$ &
$10000$
\\
OMNIGLOT &
$23000$ &
$1345$ &
$8070$
\\
CALTECH &
$4100$ &
$2264$ &
$2307$
\\
\midrule
\end{tabular}
\end{table}

\textbf{Training parameters.} Important training parameters are provided in \cref{table:architectures}. We used the Adam optimizer and exponentially decayed the initial learning rate according to the formula
\begin{align}
\alpha_t = \alpha_0 \cdot (0.1)^{\frac{t}{T - 1}},
\end{align}
where $t \in \seta{0, \ldots, T - 1}$ is the current iteration and $T$ is the total number of iterations. Early-stopping is applied according to IWAE-$5000$ evaluation on the validation set. 

\begin{table}[!h]
\setlength\tabcolsep{6pt} 
\tiny
\centering
\caption{Training parameters used for each dataset. The same architecture is used for all models, with minor modification for WNI-VAE (to account for the weight-normalization implementation). In all cases, we use a Bernoulli decoder and a Gaussian encoder. Notation: d$300$ denotes a dense layer with ELU activation and $300$ output units. z$64$ denotes 1) a dense layer with $64$ output units (represents the mean of $z$) and 2) a dense layer with softplus activation and $64$ output units (represents the variance of $z$). x$784$ denotes a dense layer with $784$ output units (represents the logits for $x$)} \label{table:architectures}
\begin{tabular}{l|c|c|c|c|c}
\toprule
Dataset & 
Encoder Architecture &
Decoder Architecture &
Initial Learning Rate &
Training Iterations &
Batch Size
\\
\midrule
MNIST (\cref{app:large_encoder})&
d$1000$-d$1000$-d$1000$-z$64$ &
d$300$-d$300$-x$784$ &
$10^{-3}$ &
$1.5 \times 10^{6}$ &
$100$
\\
MNIST &
d$300$-d$300$-z$64$ &
d$300$-d$300$-x$784$ &
$10^{-3}$ &
$1.5 \times 10^{6}$ &
$100$
\\
OMNIGLOT &
d$200$-d$200$-z$64$ &
d$200$-d$200$-x$784$ &
$10^{-3}$ &
$1.5 \times 10^{6}$&
$100$
\\
CALTECH &
d$500$-z$64$ &
d$500$-x$784$ &
$10^{-4}$ &
$4 \times 10^{5}$&
$10$
\\
\midrule
\end{tabular}
\end{table}

\textbf{Regularization strength tuning.} The denoising and weight normalization regularizers have hyperparameters $\sigma$ and $H$ respectively. See \cref{table:strengths} for hyperparameter search space details. We performed a basic grid search and tuned the regularization strength hyperparameters based on the validation set.

\begin{table}[!h]
\setlength\tabcolsep{12pt} 
\scriptsize
\centering
\caption{The regularization parameter is chosen applied based on hyperparameter tuning on the validation set. Rather than selecting for $\sigma$ or $H$ directly, we reparameterized the search space as $\sigma\cdot\sqrt{d}$ and $\frac{10}{H}$, where $d$ denotes the sample dimensionality, i.e., $\X = \R^d$. Coincidentally, we found that this reparameterization allowed us to use the same search space for both DIWAE and WNI-IWAE. We introduce the convention that setting $\frac{10}{H}$ to zero indicates setting $H = +\infty$. Via this convention, setting $\sigma\cdot\sqrt{d} = \frac{10}{H}=0$ corresponds to the standard VAE. We restricted the search space to the set $\seta{2.5, 5.0, \dots, 17.5}$, deliberately omitting $\seta{0.0}$ to not encompass the baseline (standard VAE).} 
\label{table:strengths}
\begin{tabular}{l|c|cc|cc|cc}
\toprule
& & \multicolumn{2}{c|}{MNIST} & \multicolumn{2}{c|}{OMNIGLOT} & \multicolumn{2}{c}{CALTECH} \\
& $k$ & 
$\sigma\cdot\sqrt{d}$ & $\frac{10}{H}$ &
$\sigma\cdot\sqrt{d}$ & $\frac{10}{H}$ &
$\sigma\cdot\sqrt{d}$ & $\frac{10}{H}$\\
\midrule
\multirow{3}{*}{DIWAE $(\alpha=0.5)$} 
& $ 1$ & $ 5.0$ & - & $ 2.5$ & - & $ 7.5$ & - \\
& $ 8$ & $ 5.0$ & - & $ 5.0$ & - & $ 7.5$ & - \\
& $64$ & $ 7.5$ & - & $ 5.0$ & - & $15.0$ & - \\
\midrule
\multirow{3}{*}{DIWAE $(\alpha=1.0)$} 
& $ 1$ & $ 2.5$ & - & $ 2.5$ & - & $ 5.0$ & - \\
& $ 8$ & $ 5.0$ & - & $ 5.0$ & - & $ 7.5$ & - \\
& $64$ & $ 5.0$ & - & $ 5.0$ & - & $10.0$ & - \\
\midrule
\multirow{3}{*}{WNI-IWAE} 
& $ 1$ & - & $ 5.0$ & - & $ 7.5$ & - & $ 2.5$ \\
& $ 8$ & - & $ 7.5$ & - & $ 7.5$ & - & $ 5.0$ \\
& $64$ & - & $10.0$ & - & $ 7.5$ & - & $12.5$ \\
\midrule
\end{tabular}
\end{table}

\clearpage
\section{Proofs}

\textbf{Remark.} Some of the proofs mention the notion of an infinite capacity $\F$, $\G$ or $\Q$. To clarify, we say that $\F$ has infinite capacity if it is the set of all possible functions that map from $\X$ to $\Q$. Analogously, $\G$ has infinite capacity if it is the set of all possible functions that map from $\Z$ to $\P$. We say that $\Q$ has infinite capacity if it is the set of all possible distributions over the space $\Z$.

\PropKernel*
\begin{proof}
Define $\tilde{x} = x + \veps$ and 
$\phat(x, \tilde{x}) = \phat(x)\Normal(\tilde{x} \giv x, \sigma\I)$. Rewrite the objective as
\begin{align}
R(\theta \scolon \sigma) &= \min_{f \in \F(\Q)}\Expect_{\phat(x, \tilde{x})}\brac{\KL{f(\tilde{x})}{p_\theta(z \giv x)}}\\
&\ge \Expect_{\phat(\tilde{x})}\min_{q \in \Q}\Expect_{\phat(x \giv \tilde{x})}\brac{\KL{q(z)}{p_\theta(z \giv x)}}.
\end{align}
Recall that $\F$ has infinite capacity. This lower bound is tight since we can select $f^*_\sigma \in \F$ such that
\begin{align}
f^*_\sigma(\tilde{x}) = \argmin_{q \in \Q} \Expect_{\phat(x \giv \tilde{x})} \KL{q(z)}{p_\theta(z \giv x)}.\label{eq:prekernel}
\end{align}
Reexpressing \cref{eq:prekernel} by expanding $\phat(x \giv \tilde{x})$ yields \cref{eq:kernel}.
\end{proof}

\ThmSmooth*
\begin{proof} 
Proof provided in two parts.

\textbf{Part 1.} The Kullback-Leibler divergence can be represented as a Bregman divergence associated with the strictly convex log-partition function $A$ of the minimal exponential family as follows
\begin{align}
\KL{\eta}{\eta^{(i)}} = d_A(\eta^{(i)}, \eta) = A(\eta^{(i)}) - A(\eta) - \nabla A(\eta)^\top(\eta^{(i)} - \eta).
\end{align}
Proposition 1 from \cite{banerjee2005clustering} shows that that for any convex combination weights $\set{w_i}, \sum_{i=1}^n w_i = 1$, the minimizer of a weighted average of Bregman divergences is 
\begin{align}
\sum_{i=1}^n w_i x_i = \argmin_{y \in \Omega} \sum_{i=1}^n w_i d_A(x_i, y).
\end{align}
It thus follows that 
\begin{align}
f^*_\sigma(x) 
&= \argmin_{\eta \in \Omega} \sum_{i=1}^n w_\sigma(x, x^{(i)}) \cdot \KL{\eta}{\eta^{(i)}}\\
&= \argmin_{\eta \in \Omega} \sum_{i=1}^n w_\sigma(x, x^{(i)}) \cdot d_A(\eta^{(i)}, \eta) \\
&= \sum_{i=1}^n w_\sigma(x, x^{(i)}) \cdot \eta^{(i)}.
\end{align}
\textbf{Part 2.} We will write the derivative $\nabla_{x} f^*_\sigma(x)$ in matrix form by the following notation
\begin{align*}
\nabla_x W_\sigma(x) &= \left( \begin{array}{ccc} \nabla_x w_\sigma(x, x^{(1)}) & \cdots & \nabla_x w_\sigma(x, x^{(m)}) \end{array} \right) \\
M &= \left( \begin{array}{ccc} \eta^{(1)} & \cdots & \eta^{(m)} \end{array} \right) 
\end{align*}
where we also suppose input space $x$ is $n$-dimensional, latent parameter space $\Omega$ is $d$-dimensional, and there are $m$ training examples.
Then
\[ \nabla_x f^*_\sigma(x) = M \nabla_x W_\sigma(x)^T  \]
Let $\lVert \cdot \rVert_1$ be the induced 1-norm for matrices, then by the sub-multiplicative property
\begin{align*}
\lVert \nabla_x f^*(x) \rVert_1 \leq \lVert M \rVert_1 \lVert \nabla_x W_\sigma(x)^T \rVert_1
\end{align*}
Since $\lVert M \rVert_1$ is a constant with respect to $\sigma$, we only have to bound $\lVert \nabla_x W_\sigma(x)^T \rVert_1$. To do this we study the derivative of $\nabla_x w_\sigma(x, x^{(i)})$, where
\begin{align*}
\nabla_x w_\sigma(x, x^{(i)}) &= \nabla_x \frac{K_\sigma(x, x^{(i)})}{\sum_j K_\sigma(x, x^{(j)})} \\
&= \frac{ K(x, x^{(i)})\frac{x^{(i)}-x}{\sigma^2} \sum_j K_\sigma(x, x^{(j)}) + K(x, x^{(i)})\sum_j  K(x, x^{(j)})\frac{x-x^{(j)}}{\sigma^2}}{(\sum_j K_\sigma(x, x^{(j)}))^2} \\
&= \frac{ K(x, x^{(i)}) \sum_j K_\sigma(x, x^{(j)}) \frac{x^{(i)} - x^{(j)}}{\sigma^2}}{(\sum_j K_\sigma(x, x^{(j)}))^2} 
\end{align*}
Let $|\cdot|$ denote taking element-wise absolute value, and $x \leq^* y$ denotes for all elements of the vector $|x_i| \leq |y_i|$. By Cauchy inequality and $\lVert \cdot \rVert_2 \leq \lVert \cdot \rVert_1$ we have
\begin{align*}
\nabla_x w_\sigma(x, x^{(i)}) &\leq^* \frac{K(x, x^{(i)}) \sum_j K(x, x^{(j)}) \sum_j |x^{(i)} - x^{(j)}|}{\sigma^2 (\sum_j K_\sigma(x, x^{(j)}))^2}
\leq^* \frac{1}{\sigma^2}\sum_j |x^{(i)} - x^{(j)}|
\end{align*}
Therefore 
\begin{align*}
\sup_x \lVert \nabla_x w_\sigma(x, x^{(i)}) \rVert_1 = O(1/\sigma^2)
\end{align*}
This gives us a bound on the matrix 1-norm
\begin{align*}
\sup_x \lVert \nabla_x W_\sigma(x)^T \rVert_1 \leq \sup_x \sqrt{mn} \lVert \nabla_x W_\sigma(x)^T \rVert_\infty = \sqrt{mn} \sup_x \max_{i=1}^{n} \lVert \nabla_x w_\sigma(x, x^{(i)}) \rVert_1 = O(1/\sigma^2) 
\end{align*}
Because both $\Omega$ and $\mathcal{X}$ are convex sets, this implies the following Lipschitz property
\[ \frac{\lVert f^*(x_1) - f^*(x_2) \rVert_1}{\lVert x_1 - x_2 \rVert_1} \leqslant \sup_x \lVert \nabla_x f^*(x) \rVert_1 = O(1/\sigma^2)  \]
\end{proof}

\PropStrength*
\begin{proof}
Proof is provided in two parts.

\textbf{Part 1.} Recall that $R$ is always non-negative. Since there exists $\theta \in \Theta$ such that $p_\theta(x, z) = p_\theta(z)p_\theta(x)$, and $f \in \F$ such that $f(x) = p_\theta(z)$, then $\min_\theta R(\theta\scolon \sigma) = 0$ for any choice of $\sigma$. 

\textbf{Part 2.} Let $\veps_1 \sim \Normal(\0, \sigma_1\I)$, $\veps_2 \sim \Normal(\0, \sigma_2\I)$, and $\veps_\delta = \Normal(\0, (\sigma_1 - \sigma_2)\I)$. Then
\begin{align}
R(\theta\scolon \sigma_1)
&= \min_{f \in \F} \Expect_{\veps_1}\Expect_{\phat(x)}\brac{\KL{f(x + \veps_1)}{p_\theta(z \giv x)}}\\
&= \min_{f \in \F} \Expect_{\veps_\delta}\Expect_{\veps_2}\Expect_{\phat(x)}\brac{\KL{f(x + \veps_\delta + \veps_2)}{p_\theta(z \giv x)}}\\
&\ge \Expect_{\veps_\delta}\min_{f \in \F} \Expect_{\veps_2}\Expect_{\phat(x)}\brac{\KL{f(x + \veps_\delta + \veps_2)}{p_\theta(z \giv x)}}.
\end{align}
Since $\F$ is closed under input translation,
\begin{align}
\Expect_{\veps_\delta}\min_{f \in \F} \Expect_{\veps_2}\Expect_{\phat(x)}\brac{\KL{f(x + \veps_\delta + \veps_2)}{p_\theta(z \giv x)}} = R(\theta \scolon \veps_2).
\end{align}
It thus follows that $R(\theta \scolon \sigma_1) \ge R(\theta \scolon \sigma_2)$ for all $\theta \in \Theta$.
\end{proof}

\PropDecoder*
\begin{proof}
For a given inference model $q(z \giv x)$, the optimal generator maximizes the objective
\begin{align}
\max_{g \in \G} \Expect_{\phat(x)}\Expect_{q(z \giv x)} \brac{\ln g(z)(x)} &= \max_{g \in \G} \Expect_{q(x, z)} \brac{\ln g(z)(x)}.\label{eq:g}\\
&= \max_{g \in \G} \Expect_{q(x, z)} \brac{\ln p_{g(z)}(x)} \\&\le \Expect_{q(z)} \max_{p \in \P}\Expect_{q(x \giv z)} \ln p(x)\\
&= \Expect_{q(z)} \max_{\mu \in \M}\Expect_{q(x \giv z)} \ln p_\mu(x),
\end{align}
where $p_\mu$ denotes the distribution $p \in \P$ with associate mean parameter $\mu$. This inequality is tight since we can select $g^* \in \G$ such that
\begin{align}
g^*(z) = \argmax_{\mu \in \M}\Expect_{q(x \giv z)} \ln p_\mu(x).
\end{align}
Recall that the maximum likelihood and maximum entropy solutions are equivalent for an exponential family. From the moment-matching condition of maximum entropy, it follows that
\begin{align}
g^*(z) &= \argmax_{\mu \in \M}\Expect_{q(x \giv z)} \ln p_\mu(x) \\
&= \Expect_{q(x \giv z)} \brac{T(x)} \\
&= \sum_{i=1}^n q(x^{(i)} \giv z) \cdot T(x^{(i)})\\
&= \sum_{i=1}^n \paren{\frac{q(z \giv x^{(i)})}{\sum_j q(z \giv x^{(j)})} \cdot T(x^{(i)})}.
\end{align}
\end{proof}

\PropIWAE*
\begin{proof}
Proof is provided in two parts.

\textbf{Part 1.} The relevant assumptions are that there exists $\theta \in \Theta$ such that $p_\theta(x, z) = p_\theta(z)p_\theta(x)$, and $f \in \F_H$ such that $f(x) = p_\theta(z)$. Note that $R_k$ is always non-negative. It follows readily that $\min_\theta R_k(\theta\scolon \sigma, \F_H) = 0$ for any choice of $k$. 

\textbf{Part 2.} We define $\L_k$ as
\begin{align}
\L_k &= \Expect_{\phat(x)} \ln p_\theta(x) - R_k(\theta \scolon \sigma, \F_H) \\
&= \max_{f \in \F_H} \Expect_{\phat(x)}\Expect_{\veps} \Expect_{z_1\ldots z_k \sim f(x + \veps)}\brac{\ln \frac{1}{k} \sum_{i=1}^k \frac{p_\theta(x, z_i)}{f(x + \veps)(z_i)}}.
\end{align}
It suffices to show that $\L_{k} \ge \L_{m}$ when $k \ge m$.
We adapt the proof from \cite{burda2015importance}'s Theorem 1 as follows. Let $|I| = m$ denote a uniformly distributed subset of distinct indices from $\seta{1, \ldots, k}$. For any choice of $f \in \F_H$, the following inequality holds
\begin{align}
&\Expect_{\phat(x)}\Expect_{\veps} \Expect_{z_1\ldots z_k \sim f(x + \veps)}\brac{\ln \frac{1}{k} \sum_{i=1}^k \frac{p_\theta(x, z_i)}{f(x + \veps)(z_i)}}\\
&= \Expect_{\phat(x)}\Expect_{\veps} \Expect_{z_1\ldots z_k \sim f(x + \veps)}\brac{\ln \Expect_{I=\seta{i_1\ldots i_m}}\brac{\frac{1}{m} \sum_{j=1}^m \frac{p_\theta(x, z_{i_j})}{f(x + \veps)(z_{i_j})}}}\\
&\ge \Expect_{\phat(x)}\Expect_{\veps} \Expect_{z_1\ldots z_k \sim f(x + \veps)}\Expect_{I=\seta{i_1\ldots i_m}}\brac{\ln \frac{1}{m} \sum_{j=1}^m \frac{p_\theta(x, z_{i_j})}{f(x + \veps)(z_{i_j})}}\\
&= \Expect_{\phat(x)}\Expect_{\veps} \Expect_{z_1\ldots z_m \sim f(x + \veps)}\brac{\ln \frac{1}{m} \sum_{i=1}^m \frac{p_\theta(x, z_{i})}{f(x + \veps)(z_{i}))}}.
\end{align}
It thus follows that $\L_k \ge \L_m$.
\end{proof}

\end{document}